\documentclass{article}
\pdfoutput=1

\usepackage[final]{neurips_2022}

\usepackage[utf8]{inputenc} % allow utf-8 input
\usepackage[T1]{fontenc}    % use 8-bit T1 fonts
\usepackage{hyperref}       % hyperlinks
\usepackage{url}            % simple URL typesetting
\usepackage{booktabs}       % professional-quality tables
\usepackage{amsfonts}       % blackboard math symbols
\usepackage{amsmath}
\usepackage{nicefrac}       % compact symbols for 1/2, etc.
\usepackage{microtype}      % microtypography
\usepackage{xcolor}         % colors
\usepackage{graphicx}
\usepackage{hyperref}
\usepackage{enumitem}
\usepackage{amssymb}
\usepackage{url}
\usepackage{soul}
\usepackage[english]{babel}
\usepackage{amsthm}
\usepackage{algorithm}
\usepackage{algpseudocode}
\usepackage{graphicx}
\usepackage{natbib}

\DeclareMathOperator*{\argmax}{arg\,max}
\DeclareMathOperator*{\argmin}{arg\,min}
\DeclareMathOperator*{\maxi}{max}

\newcommand{\vx}{\mathbf{x}}

\newtheorem{theorem}{Theorem}[section]
\newtheorem{definition}{Definition}
\newtheorem{problem}{Problem}

\newtheorem{lemma}{Lemma}

\newcommand{\algrule}[1][.2pt]{\par\vskip.5\baselineskip\hrule height #1\par\vskip.5\baselineskip}

\title{Certification of Distributional Individual Fairness}

\author{%
  Matthew Wicker \\
  The Alan Turing Institute \\
  \texttt{mwicker@turing.ac.uk} \\
  \And
  Vihari Piratia \\
  University of Cambridge \\
  \texttt{vp421@cam.ac.uk} \\
  \And
  Adrian Weller \\
  University of Cambridge \& \\
  The Alan Turing Institute \\
}

\begin{document}

\maketitle

\begin{abstract}
Providing formal guarantees of algorithmic fairness is of paramount importance to socially responsible deployment of machine learning algorithms. In this work, we study formal guarantees, i.e., certificates, for individual fairness (IF) of neural networks. We start by introducing a novel convex approximation of IF constraints that exponentially decreases the computational cost of providing formal guarantees of local individual fairness. We highlight that prior methods are constrained by their focus on global IF certification and can therefore only scale to models with a few dozen hidden neurons, thus limiting their practical impact. We propose to certify \textit{distributional} individual fairness which ensures that for a given empirical distribution and all distributions within a $\gamma$-Wasserstein ball, the neural network has guaranteed individually fair predictions. Leveraging developments in quasi-convex optimization, we provide novel and efficient certified bounds on distributional individual fairness and show that our method allows us to certify and regularize neural networks that are several orders of magnitude larger than those considered by prior works. Moreover, we study real-world distribution shifts and find our bounds to be a scalable, practical, and sound source of IF guarantees. 
\end{abstract}

\section{Introduction}

There is a growing concern about the potential of machine learning models to perpetuate and amplify discrimination \citep{barocas2016big}. Machine learning algorithms have been put forward to automate decision making in a variety of fairness-sensitive domains, such as healthcare \citep{davenport2019potential}, employment \citep{ding2021retiring}, and criminal justice \citep{dressel2018accuracy, zilka2022survey}. It has been demonstrated that such models produce biased outcomes that unfairly disadvantage certain individuals or groups \citep{seyyed2021underdiagnosis, yurochkin2020sensei}. To combat this, there has been a surge of interest in algorithmic fairness metrics \citep{mehrabi2021survey}. Fairness metrics provide practitioners with a means to quantify the extent of bias within a model and facilitate development of heuristics for debiasing models \citep{madras2018learning, yurochkin2019training}. However, relying solely on heuristic debiasing methods may not be satisfactory for justifying the deployment of machine learning models as they do not provide guarantees of fairness. To address this concern, there has been a recent focus on certification approaches, which provide formal guarantees that a model adheres to a specified fairness criterion \citep{john2020verifying, benussi2022individual, khedr2022certifair}. 
Providing certified guarantees of fairness is of utmost importance as it offers users, stakeholders, and regulators formal assurances that the predictions of a model align with a rigorous standards of fairness. These guarantees serve as a powerful tool to promote trust, accountability, and ethical deployment of machine learning models.

%One key metric of fairness for deep learning systems is \textit{individual fairness} which enforces the intuitive property that for all \textit{similar} individuals the neural network (NN) issues similar predictions. Individuals are similar if they differ only with respect to protected attributes or features correlated with the protected attributes. This similarity is captured by a fair distance metric, defined formally in Section~\textbf{?}. Recent works have tackled the key challenge of certifying individual fairness for neural networks. However, they fail to certify and train debiased neural networks with more than a few dozen hidden neurons. 

In this work, we study the problem of certifying  \textit{individual fairness} (IF) in neural networks (NNs). Individual fairness enforces the intuitive property that a given neural network issues similar predictions for all pairs of \textit{similar} individuals \citep{dwork2012fairness}. Individuals are considered similar if they differ only with respect to protected attributes (e.g., race, gender, or age) or features  correlated with the protected attributes. This similarity is captured by a fair distance metric which can be learned by querying human experts or extracted from observed data \citep{mukherjee2020two}. Given a NN and a fairness metric, recent works have established procedures for certifying that a NN conforms to a given IF definition, which is a crucial step in developing and deploying fair models \citep{benussi2022individual, khedr2022certifair}. %However, these approaches only scale to NNs with a few dozen neurons, a limitation that relegates IF guarantees to only the simplest problems. Current approaches focus on scale only to 
While current approaches to IF guarantees for NNs are effective for simple problems, they face scalability issues when dealing with NNs containing more than a few dozen neurons. This limitation is due to their emphasis on global individual fairness, which guarantees the IF property for every input within the NN's domain \citep{benussi2022individual, khedr2022certifair}. Although global IF certification is the gold standard, it places constraints on extreme outliers, such as ensuring loan decisions are fair to a three-year-old billionaire applying for a small business loan. Subsequently, as the dimension of the input space grows, solving an optimization problem that covers the entire domain becomes either impossible for current solvers to deal with or computationally prohibitive. 
%Scalability of these methods is hindered by their focus on global individual fairness, which ensures that the fairness property holds for any input in the entire domain of the NN. Global IF certification, while a gold standard, enforces constraints on unrealistic individuals (e.g., a 3 year old billionaire applying for a small business loan). Therefore, we focus on more realistic constraints on \textit{distributional individual fairness} (DIF) with the goal of enabling practical and formal guarantees of IF that can justify a model's socially responsible deployment.
%
On the other hand, distributional individual fairness (DIF), first proposed in \cite{yurochkin2019training}, enforces that a model's predictions are individually fair w.r.t. a family of distributions that are within a $\gamma-$Wasserstein ball of the empirical distribution over observed individuals. %While the work of \cite{yurochkin2019training} provides heuristic approaches for debiasing , 
The focus on distributions removes the constraints on unrealistic individuals while enforcing fairness on the relevant part of the input domain. As such, providing certification of DIF can provide strong guarantees of fairness that scale to even relatively large neural network models. However, prior works in DIF focus only on heuristic debiasing techniques \citep{yurochkin2019training}. 

In this work, we provide the first formal certificates of DIF. As it cannot be computed exactly, we present a framework for bounding distributional individual fairness. We start by presenting a novel convex relaxation of local IF constraints, i.e., IF with respect to a single prediction. Our approach to local IF certification offers an exponential computational speed-up compared with prior approaches. %when computing local IF certificates and enables real-time auditing of individual fairness. %In this work, we present the first method tailored for certifying distributional individual fairness. In order to do so, we first propose a novel convex relaxation of local individual fairness with orthotope constraints in the input space, enabling efficient upper and lower bounds on local individual fairness. 
Building on this efficient bound, we are able to propose an optimization problem whose solution, attainable due to the problem's quasi-convexity, produces sound certificates of DIF. %This method provides the first formal guarantees on DIF. 
Further, our certified bounds can be easily incorporated as a learning objective, enabling efficient individual fairness training.
In a series of experiments, we establish that our method is able to certify that IF holds for meaningful, real-world distribution shifts. Further, our proposed regularization is able to certifiably fair neural networks that are two orders of magnitude larger than those considered in previous works in a fraction of the time.
%We implement our certification and training algorithms and test them across a series of well-known fairness benchmarks. We highlight that the efficiency of our bounds enables us to reduce the computational time to with local individual fairness constraints from 10 hours on a multi-GPU server to less than 5 minutes on a laptop. Moreover, the efficiency of our method allows us to scale to neural networks fifty times larger than previous methods thereby allowing training of individually fair neural networks on considerably larger datasets with considerably larger neural networks. Finally, we validate the soundness and effectiveness of our distributional individual fairness bounds on a by studying datasets on a variety of real-wold distribution shifts. We find that our bounds tightly bound the effect of such distribution shifts and can therefore be reliable as a metric of individual fairness at deployment time.
We highlight the following key contributions of this work:
\begin{itemize}[leftmargin=*]
    \item We prove a convex relaxation of local individual fairness, enabling exponentially more efficient local IF certification compared with prior mixed integer linear programming approaches.
    \item We formulate a novel optimization approach to certificates of local and distribution individual fairness and leverage quasi-convex optimization to provide the first formal upper and lower bounds, i.e., certificates, on distributional individual fairness. %The first time such bounds have been leveraged to produce formal certificates
    \item We empirically demonstrate the tightness of our distributional individual fairness guarantees compared with real-world distributional shifts, and demonstrate the considerable scalability benefits of our proposed method.
\end{itemize}

%Measures roughly fall into two categories, group and individual fairness metrics.  In this work we study various notions of individual fairness which enforces the intuitive property that for all \textit{similar} individuals the neural network (NN) issues similar predictions. Individuals are considered similar accoriding to a \textit{fair distance metric} which assigns a low distance between individuals who differ with respect to protected attributes (e.g., race, gender, age) or any features highly correlated with said protected attributes, e.g., area code can be highly correlated with race.  

\section{Related Work}\label{sec:relatedworks}

Individual fairness, originally investigated by \citet{dwork2012fairness}, establishes a powerful and intuitive notion of fairness. Subsequent works have focused on defining the individual fairness metrics \citep{mukherjee2020two, yurochkin2020sensei}, expanding the definition to new domains \citep{gupta2021individual, xu2022gfairhint, doherty2023individual}, and importantly, guaranteeing that a model conforms to individual fairness \citep{john2020verifying, ruoss2020learning, benussi2022individual, khedr2022certifair, peychev2022latent}. In \citep{john2020verifying} the authors consider a relaxation of individual fairness and present methods to verify that it holds for linear classifiers. In \citep{yeom2020individual, peychev2022latent} the authors extend randomized smoothing to present statistical guarantees on individual fairness for neural networks. These guarantees are much weaker than the sound guarantees presented in \cite{benussi2022individual, khedr2022certifair} which are based on mixed integer linear programming (MILP). Both of these methods focus solely on the global notion of individual fairness which proves to only scale to neural networks with a few dozen hidden neurons. Our work relaxes the need to certify a property globally allowing us to scale to neural networks that are orders of magnitude larger than what can be considered by prior works. 
%In \cite{ruoss2020learning} the authors also use MILP to train locally certifiably individually fair representations in NLP. Which is satisfactory for understanding if a single prediction is fair, but MILP formulations are computationally costly. In contrast, the method presented by this paper is exponentially faster. 
Further, prior works have sought to debias models according an individual fairness criterion \citep{yurochkin2019training, benussi2022individual, khedr2022certifair}. In \citet{yurochkin2019training} the authors present a method for debiased training relying on PGD which is insufficient for training a NN with strong fairness guarantees \citep{benussi2022individual}. In \citep{benussi2022individual, khedr2022certifair} the authors use linear programming formulations of fairness during training which leads to strong IF guarantees but at a large computational cost and greatly limited scalability.  %The certification methods in this work can be seen as a more local and approximate,  but exponentially more efficient version of the results presented in \cite{benussi2022individual}.

Distributional individual fairness was originally proposed by \cite{yurochkin2019training} and the importance of distributional robustness of notions of fairness has been underscored in several recent woks \citep{sagawa2019distributionally, sharifi2019average, taskesen2020distributionally, mukherjee2022domain}. To the best of our knowledge, this work is the first that provides formal certificates that guarantee a model satisfies distributional individual fairness. We provide further discussion of works related to distributional robustness in Appendix~\ref{appendix:relatedworks}.
%There are several works that study the problem of fairness under data-set shifts. In \cite{lan2017discriminatory} the authors find that standard transfer learning improves accuracy at the cost of fairness. In \cite{ding2021retiring} the authors show that a model that is fair in one geographic context may be unfair in another.  %In \cite{kallus2018residual} shows that poor quality data on a given subgroup can lead to unfairness even in bias-mitigated models.
%In \cite{schrouff2022maintaining} the authors use a causal model to understand when and how fairness will degrade in different contexts. 
%In \cite{schumann2019transfer}, the authors study how to use representation learning techniques to perform transfer learning such that the model maintains fairness, but unlike this work they must modify the NN and do not offer provable certificates of fairness. In \cite{schrouff2022maintaining} the authors collect many related works on fairness adaptation and show how, through a causal lens, many of them introduce assumptions that do not hold in practice. In contrast, the method presented here can guarantee individual fairness for a given  classifier and individual with non-restrictive assumptions on the form of the neural network. 
\section{Background}

%In this section, we define individual fairness and cover prior works for training individually fair neural networks. 
%We consider a supervised learning scenario where we are given a dataset of $n_{\mathcal{D}}$-many inputs and labels, $\mathcal{D} = \{x^{(i)}, y^{(i)}\}_{i=1}^{n_{\mathcal{D}}}$, with inputs $x^{(i)} \in \mathbb{R}^n$, and corresponding target outputs $y^{(i)} \in \mathbb{R}^{m}$. 
% The tasks studied here are all binary classification tasks, thus $y^{(i)} \in \{0,1\}$.
We consider a supervised learning scenario in which we have a set of $n$ feature-label pairs drawn from an unknown joint distribution $\{(x^{(i)}, y^{(i)})\}_{i=1}^{n} \sim P^{0}(x,y)$ with $x \in \mathbb{R}^{m}$ and $y \in \mathbb{R}^{k}$. 
We consider a feed forward neural network (NN) as a function $f^\theta:\mathbb{R}^{m}\to\mathbb{R}^k$, parametrised by a vector  $\theta \in \mathbb{R}^{p}$. % containing all the weights and biases of the network. 
We define a local version of the individual fairness definition that is studied in \citep{john2020verifying, benussi2022individual} to evaluate algorithmic bias:
\begin{definition}\textbf{Local $\delta$-$\epsilon$-Individual Fairness}\label{def:IFdef} Given $\delta > 0$ and $\epsilon \geq 0$, we say that $f^{\theta}$ is locally individually fair (IF) with respect to fair distance $d_{\text{fair}}$ and input $x'$ iff: 
$$ \forall x'' \ s.t. \ d_{\text{fair}}(x', x'') \leq \delta \implies |f^{\theta}(x') - f^{\theta}(x'')| \leq \epsilon $$
{\upshape For a given value of $\delta$ we use the notation  $\mathcal{I}(f^{\theta}, x, \delta)$ to denote the function that returns the largest value of $\epsilon$ such that the local IF property holds.  Further, $\mathcal{I}(f^{\theta}, x, \delta) = 0$ holds if a network is perfectly locally individually fair. Throughout the paper, we will refer to the value $\epsilon$ as the IF \textit{violation} as this is the amount by which the perfect fairness property is violated.}
\end{definition}
Several methods have been investigated for crafting fair distance metrics \citep{mukherjee2020two}. Methods in this work can flexibly handle a majority of the currently proposed metrics. We discuss the particulars in Appendix~\ref{appendix:computations}. While certifying local IF is sufficient for proving that a given prediction is fair w.r.t. a specific individual, it does not suffice as proof that the model will be fair for unseen individuals at deployment time. To this end, prior works have investigated a global individual fairness definition. Global $\epsilon$-$\delta$ individual fairness holds if and only if $\forall x \in \mathbb{R}^{m}, \mathcal{I}(f^{\theta}, x, \delta) \leq \epsilon$. %If a neural network satisfies global individual fairness, then one can be sure that for any individual seen at deployment time the neural network's prediction will be fair. 
However, global individual fairness is a strict constraint and training neural networks to satisfy global individual fairness comes at a potentially prohibitive computational cost. For example, training a neural network with only 32 hidden units can take up to 12 hours on a multi-GPU machine \citep{benussi2022individual}. 

\section{Distributional Individual Fairness}

In this section, we describe \textit{distributional individual fairness} (DIF) as a notion that can ensure that neural networks are individually fair at deployment time while not incurring the considerable overhead of global training and certification methods. DIF ensures that individuals in the support of $P^{0}$ and distributions close to $P^{0}$ are treated fairly. To measure the distributions close to $P^{0}$ we consider the $p$-Wasserstein distance between distributions, $W_p(P^{0}, \cdot)$. Further, we denote the ball of all distributions within $\gamma$ $p$-Wasserstein distance of $P^{0}$ as $\mathcal{B}_{\gamma, p}(P^{0})$.  To abbreviate notation, we drop the $p$ from the Wasserstein distance, noting that any integer $p \geq 1$ can be chosen.
\begin{definition}\textbf{Distributional $\gamma$-$\delta$-$\epsilon$-Individual Fairness }\label{def:DIFdef} Given individual fairness parameters $\delta > 0$ and $\epsilon \geq 0$, and distributional parameter $\gamma > 0$, we say that $f^{\theta}$ is distributionally individually fair (DIF) with respect to fair distance $d_{\text{fair}}$ if and only if the following condition holds: 
\begin{align}
    \sup_{P^{\star} \in \mathcal{B}_{\gamma}(P^{0})} \bigg( \mathbb{E}_{x \sim P^{\star}}[\mathcal{I}(f^{\theta}, x, \delta)] \bigg) \leq \epsilon 
    \label{eq:difconstraint1}
    %&\argmax_{x \in \text{supp}(P^{\star})} \leq c \epsilon \label{eq:difconstraint2}
\end{align}
\end{definition}
Intuitively, distributional individual fairness ensures that for all distributions in a Wasserstein ball around $P^{0}$ we maintain individual fairness on average. In Figure~\ref{fig:examplefig}, we visualize the difference between local, distributional, and global IF in a toy two dimensional settings. Prior work that focuses on DIF only provides a heuristic method based on the PGD attack to train debiased NNs and therefore provide no formal guarantees \citep{yurochkin2019training}. Below we state two key computational problems in the study of guarantees for distributional individual fairness:
\begin{problem}\label{prob:certification} Certification of Distributional Individual Fairness:
Given a neural network, $f^{\theta}$ and parameters $\delta$ and $\gamma$, compute a value $\overline{\epsilon}$ such that $\overline{\epsilon}$ provably upper-bounds the left hand side of Equation~\eqref{eq:difconstraint1}. An optimal certificate is the smallest such  $\overline{\epsilon}$ such that Equation~\eqref{eq:difconstraint1} holds.
\end{problem}
The problem of certifying distributional individual fairness is critical to ensuring that the definition can be used in practice as it 
%Certification for DIF 
allows for regulators and auditors to have similar assurances to global individual fairness, but only on what is deemed to be the relevant portion of the input space as controlled by $\gamma$. 
\begin{problem}\label{prob:training} Distributionally Individually Fair Training:
Given a randomly initialized neural network, $f^{\theta}$ and parameters $\delta$ and $\gamma$ and a dataset drawn from $P^{0}$, learn a neural network that approximately minimizes both the standard loss, $\mathcal{L}$, and the DIF violation: 
\begin{align}\label{eq:learningproblem}
    \argmin_{\theta \in \Theta} \mathbb{E}_{P^{0}}\big[\mathcal{L}(f^{\theta}, x, y)\big] + \sup_{P^{\star} \in \mathcal{B}_{\gamma, p}(P^{0})} \bigg( \mathbb{E}_{P^{\star}}[\mathcal{I}(f^{\theta}, x, \delta)] \bigg)
\end{align}
\end{problem}
Problem 2 is tantamount to certified debiased training of a neural network w.r.t the DIF definition. Certified debiasing places a focus on formal guarantees that is not present in prior distributional individual fairness work. %We highlight that unlike the approach proposed by \citep{yurochkin2019training}, our approach optimizes the $\mathcal{I}$ function as we are focused on learning NNs with a posteriori certifiable individual fairness whereas the method proposed in \cite{yurochkin2019training} optimizes only a heuristic lower-bound on individual fairness. %However, regularization of the second constraint based on the $\gamma$ Wasserstein ball cannot be straight-forwardly achieved with previous methods, especially in light of the fact that we would like a posteriori certified guarantees on the quality of the regularized network.  

\begin{figure}
    \centering
\includegraphics[width=0.95\textwidth]{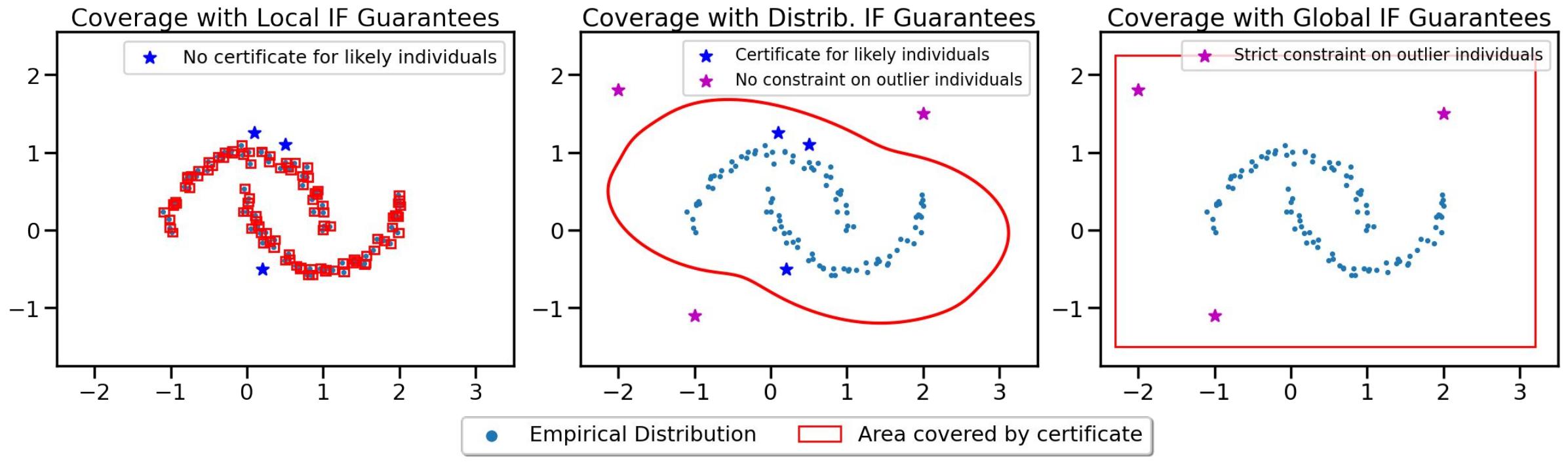}
    \caption{Using the halfmoons dataset we illustrate the benefits of DIF. \textbf{Left:} local IF does not do enough to provide guarantees on individuals we are likely to see. \textbf{Right:} Global IF provides a strict constraint on anomalous points, leading to scalability issues. \textbf{Center}: Distributional IF covers likely individuals without over-constraining the problem leading to scalable and reliable fairness guarantees.}
    \label{fig:examplefig}
\end{figure}

\section{Methodology}

To address Problems~\ref{prob:certification} and \ref{prob:training} we start by providing a novel and efficient computation for certificates of local individual fairness. We then build on this efficient guarantee to compute formal bounds and proposed regularizers for certified DIF training. %address Problems~\ref{prob:certification} and \ref{prob:training}. 
%In this section, we start by providing a novel and efficient solution to computing certified bounds on local individual fairness. Using this bound we present solutions to Problems~\ref{prob:certification} and \ref{prob:training}. 

\subsection{Certifying Local Individual Fairness}

A key step in certification of distributional individual fairness is proving that a given neural network satisfies local individual fairness i.e., Definition~\ref{def:IFdef}. While there are many methods for efficiently certifying local robustness properties, individual fairness must be certified w.r.t. a ball of inputs whose shape and size is determined by $d_{\text{fair}}$. Existing works in crafting individual fairness metrics use Mahalanobis metrics to express the fair distance between two individuals, i.e., $d_{\text{fair}}(x, x') = \sqrt{(x-y)^\top S^{-1}(x-y)}$ where $S$ is a positive semi-definite matrix $\in \mathbb{R}^{m \times m}$. In order to certify this more expressive class of specifications prior works use linear bounding techniques as part of a larger MILP formulation \citep{benussi2022individual}. In this work, we show that one can in fact use exponentially more efficient bound propagation techniques developed for robustness certification to certify local individual fairness. In order to do so, we provide the following bound that over-approximates the $d_{\text{fair}}$ metric ball with an orthotope: 
\newpage
\begin{theorem}\label{lem:intervalboundlemma}
Given a positive semi-definite matrix $S \in \mathbb{R}^{m \times m}$, a feature vector $x' \in \mathbb{R}^{m}$, and a similarity theshold $\delta$, all vectors $x'' \in \mathbb{R}^m$ satisfying $d_S(x', x'') \leq \delta$ are contained within the axis-aligned orthope:
$$ \Big[ x' - \delta\sqrt{d}, x' + \delta\sqrt{d} \Big] $$
where $d = diag(S)$, the vector containing the elements along the diagonal of $S$.
\end{theorem}
%\vspace{-1em}
\iffalse
\begin{proof}
Consider the case of $\delta = 1$. The width of the desired interval along the $i^{th}$ dimension can be obtained by solving $\max_{x''} e_i^T(x'' - x')\ s.t.\ d_{S}(x', x'') \leq 1$ where $e_i$ is the $i^{th}$ canonical basis vector. Let $S = R^T R$ where $R$ is the matrix square root of $S$, we can re-write the optimization in the rotated space by change of variables, $u:=R^{T}(x'' - x')$. W.r.t. $u$ the optimization becomes $\max_{u}R^T_i u \ \  s.t.\ ||u|| \leq 1$ where $R_i$ is the $i^{th}$ column of $R$. The solution of this optimization problem is then $R_i^T R_i / ||R_i|| = \sqrt{R_i^T R_i} = \sqrt{S}_{i,i}$. Generalizing to each dimension $i$ we have the bound $\sqrt{diag(S)}$, as desired. Because $S^{-1}$ is a linear transformation this bound remains sound when scaled by $\delta$ or translated to be centered at an arbitrary feature vector $x$. Proof of results related to Theorem~\ref{lem:intervalboundlemma} can be found in ~\cite{emrich2013optimal, MahaBound}.
\end{proof}
\fi
%
Theorem~\ref{lem:intervalboundlemma} holds as well  for weighted $\ell_p$ fairness metrics by setting the diagonal entries of $S$ to the weights for metric provides an over approximation of the $\ell_\infty$ metric, which is also a valid over-approximation for any $p \leq \infty$.  
The key consequence of Theorem~\ref{lem:intervalboundlemma} is that it directly enables efficient certification of local $\delta$-$\epsilon$ individual fairness. By over-approximating the fair distance with an axis-aligned orthotope, we are able to leverage the efficient methods developed for robustness certification \citep{gowal2018effectiveness}. Such methods allow us to take the given input interval and compute output bounds $[y^{L}, y^{U}]$ such that $\forall x \in [x' - \delta\sqrt{d}, x' + \delta\sqrt{d}], f^{\theta}(x) \in [y^{L}, y^{U}]$. Subsequently, it is clear to see that if $|y^{U} - y^{L}| \leq \epsilon$ then we have certified that $\forall x' \in [x - \delta\sqrt{d}, x + \delta\sqrt{d}] \implies |f^{\theta}(x) - f^{\theta}(x')| \leq \epsilon$ which proves that local individual fairness holds according to Definition~\ref{def:IFdef}. Computing such output bounds requires the computational complexity of two forward \cite{mirman2018differentiable, gowal2018effectiveness}. MILP approaches, on the other hand, requires complexity cubic in the largest number of neurons in a layer plus requires exponentially (in the number of problem variables) many MILP iterations \cite{benussi2022individual}. 
We denote the function that takes the interval corresponding to an individual $x$ and value $\delta$ and returns $|y^{U} - y^{L}|$ with the notation $\overline{\mathcal{I}(f^{\theta}, x, \delta)}$. Given that bound propagation methods are sound, this value over-approximates the $\epsilon$ function in Definition~\ref{def:IFdef}. We emphasize that this approach to local fairness certification only takes two forward passes through the neural network, and can therefore be used as real-time decision support. We provide exact details of the computations of this upper bound in Appendix~\ref{appendix:computations}. %It is clear that a for a given individual, $x'$, and neural network, $f^{\theta}$, proving that $\forall x'' \in [x' - \delta\sqrt{d}, x' + \delta\sqrt{d}] \implies |f^{\theta}(x) - f^{\theta}(x')| \leq \epsilon$ is sound proof that local $\delta$-$\epsilon$-IF is satisfied by $f^{\theta}$ at $x'$. Moreover, because this is a computation over intervals the efficient methods developed for robustness certification can all be leveraged in this setting \citep{gowal2018effectiveness}. 
By exploring the orthotope using methods developed for adversarial attacks, e.g., PGD \citep{madry2017towards}, we can find an input $x^{\star}$ that approximately maximizes $|f^{\theta}(x') - f^{\theta}(x^{\star})|$. Subsequently, this value is a valid lower-bound on the local IF violation around $x$. As such, we will denote the value produced by PGD as \underbar{$\mathcal{I}(f^{\theta}, x, \delta)$}.% , as PGD attacks represent a valid lower bound on the $\epsilon$ function in Definition~\ref{def:IFdef}.

%\textbf{Visualizing Distributional IF} Before discussing certification of the DIF property, we highlight that computing the local certificate, $\overline{\mathcal{I}(f^{\theta}, x, \delta)}$, requires only two forward passes through the network architecture, details provided in Appendix~\ref{appendix:computations}. As such, it is suitable for real time prediction monitoring. Such monitoring can overcome a weakness of DIF. Unlike global IF, DIF does not cover the entire domain, visualized in Figure~\ref{fig:examplefig}. While this comes with many advantages, i.e., efficiency and scalability, one cannot completely rule out anomalous individuals, purple stars in Figure~\ref{fig:examplefig}. While DIF does not consider such individuals at certification time, our efficient local IF certificates can be used as real-time prediction auditing to ensure that all individuals are treated fairly at inference time.

\subsection{Certifying Distributional Individual Fairness}

In this section we leverage the efficient upper bound on $\mathcal{I}(f^{\theta}, x, \delta)$ to prove bounds on distributional individual fairness. We first highlight that computing the Wasserstein distance between two known distributions is non-trivial, and that in general we do not have access to the joint distribution $P^{0}$. Thus, computing the Wasserstein ball around the data generating distribution is infeasible. To overcome this, as is traditionally done in the distributional robustness literature, we make the practical choice of certifying distributional individual fairness w.r.t. the empirical distribution $\hat{P}^{0}$ which is the probability mass function with uniform density over a set of $n$ observed individuals \citep{sinha2017certifying}. While this introduces approximation in the finite-sample regime, we can bound the absolute difference between our estimate and the true value by using Chernoff's inequality, discussion in Appendix~\ref{appendix:computations}. %of $\epsilon^{\star}$ with concentration inequalities. Using Chernoff's inequality, we can enforce that a finite sample estimate is within absolute error $\tau$ with probability $1- \lambda$ by ensuring that the number of samples $n$ is at least $log(\lambda/2)/(-2\tau^2)$, discussion in Appendix~\ref{appendix:computations}.
%to be within $\tau$ of the true expectation ( l.h.s of Equation~\ref{eq:optimproblem}) with probability $1- \lambda$, we need only to ensure that $n$ is at least $log(\lambda/2)/(-2\tau^2)$ as prescribed by Chernoff's inequality, discussion in Appendix~\ref{appendix:computations}.  
Given the empirical distribution of individuals as well as IF parameters $\gamma$ and $\delta$, we pose an optimization objective whose solution is the tightest value of $\epsilon$ satisfying Definition~\ref{def:DIFdef}.
%
\iffalse
\begin{subequations}
\begin{align}\label{eq:optimproblem}
    \begin{split}
    \text{maximize}\quad \dfrac{1}{n}\sum_{i=1}^{n} &\mathcal{I}(f^{\theta}, x^{(i)} + \phi^{(i)}, \delta)
    \end{split}
    \begin{split}\label{eq:validconstraint}
    &s.t.\ \ \phi^{(i)} \in \mathbb{R}^{m},
    \end{split}\\
    \begin{split}\label{eq:gammaconstraint}
    &\dfrac{1}{n}\sum_{i=1}^{n} || \phi^{(i)} ||^p \leq \gamma^p
    \end{split}
\end{align}
\end{subequations}
\fi
\begin{align}\label{eq:optimproblem}
    \epsilon^{*} = \maxi_{\phi^{(1..n)}}\quad \dfrac{1}{n}\sum_{i=1}^{n} &\mathcal{I}(f^{\theta}, x^{(i)} + \phi^{(i)}, \delta),\quad  s.t., \quad \phi^{(i)} \in \mathbb{R}^{m},\ \dfrac{1}{n}\sum_{i=1}^{n} || \phi^{(i)} ||^p \leq \gamma^p
\end{align}
%The solution, $\epsilon^{\star}$, corresponds to the optimal value of $\epsilon$ that satisfying Definition~\ref{def:DIFdef}. 
Unfortunately, computing the optimal values of $\phi^{(i)}$ even when $n = 1$ % for even a single atom in the support of $\hat{P}^{0}$ 
is a non-convex optimization problem which is known to be NP-hard \citep{katz2017reluplex}. We restate this optimization problem in terms of our bounds on $\mathcal{I}(f^{\theta}, x, \delta)$ to get formal lower (\underline{$\epsilon$}) and upper ($\overline{\epsilon}$) bounds on $\epsilon^{\star}$ such that \underline{$\epsilon$} $\leq \epsilon^{\star} \leq \overline{\epsilon}$.

\subsection{Bounding Distributional Individual Fairness} 
%In this section, we provide upper and lower bounds on the maximization in Equation~\eqref{eq:optimproblem}. 
\vspace{-0.5em}
\paragraph{Lower-bound on DIF}
Observe that any selection of $\{\phi^{(i)} \}_{i=1}^{n}$ satisfying the constraints represents a feasible solution to the optimization problem and therefore a lower-bound on the maximization problem of interest. Thus, the lower-bound can be stated as:
%By leveraging a first-order optimization of our local individual fairness certification we can arrive at lower-bounds on the distributional individual fairness violation.  
\begin{align}\label{eq:lowerbound}
\text{\underbar{$\epsilon$}} = \maxi_{\phi^{(1..n)}}\quad \dfrac{1}{n}\sum_{i=1}^{n} &\ \text{\underbar{$\mathcal{I}(f^{\theta}, x^{(i)} + \phi^{(i)}, \delta)$}} \quad  s.t., \quad \phi^{(i)} \in \mathbb{R}^{m},\ \dfrac{1}{n}\sum_{i=1}^{n} || \phi^{(i)} ||^p \leq \gamma^p
\end{align}
%Where this bound is a valid lower-bound for any feasible, non-optimal assignment of $\phi$. Computing the optimal assignment of $\phi$ would upper-bound $\epsilon^{*}$; however, given the non-convexity of the 
We can optimize this lower bound to be tight by observing that for any given $i$, the function \underbar{$\mathcal{I}(f^{\theta}, x^{(i)} + \phi^{(i)}, \delta)$} is differentiable w.r.t. $\phi^{(i)}$ and is therefore amenable to first-order optimization. In particular, we denote $\phi^{(i)}_{0}$ to be a randomly selected, feasible assignment of $\phi^{(i)}_{0}$, we can then gradient ascent to find a locally optimal assignment of $\phi^{(i)}$:
$$ \phi^{(i)}_{j+1} \gets \phi^{(i)}_{j} + \alpha \nabla_{\phi^{(i)}} \text{\underbar{$\mathcal{I}(f^{\theta}, x^{(i)} + \phi^{(i)}, \delta)$}} $$
where $\alpha$ is a learning rate parameter. Subsequently, one could use a projection step in order to ensure that the constraint in Equation~\eqref{eq:optimproblem} is never violated. We note that this is a strict lower-bound unless the Equation~\eqref{eq:lowerbound} is solved exactly and \underbar{$\mathcal{I}(f^{\theta}, x^{(i)} + \phi^{(i)}, \delta)$} $= \mathcal{I}(f^{\theta}, x^{(i)} + \phi^{(i)}, \delta) $ for all $i$. Finally, we denote the computed solution to Equation~\eqref{eq:lowerbound} as \underline{$I_\gamma(f^{\theta}, X, \delta)$} where $\gamma$ is the distributional parameter of the DIF specification and $X$ is the tensor containing the $n$ individual feature vectors. %Unless solved optimally, this optimization yields a valid lower-bound on the DIF violation.

\paragraph{Upper-bound on DIF}
An upper-bound on the maximization in Equation~\eqref{eq:optimproblem}, i.e., $\overline{\epsilon} \geq \epsilon^{\star}$, constitutes a certificate that the model provably satisfies distributional individual fairness for all $\epsilon$ greater than $\overline{\epsilon}$. Given that the optimization posed by Equation~\eqref{eq:optimproblem} is highly non-convex, searching for an optimal solution cannot be guaranteed to converge globally without knowledge of global Lipschitz constants, which are difficult to compute efficiently for large NNs \citep{fazlyab2019efficient}. Instead, we transform the optimization problem from one over the input space to the space of over-approximate local IF violations: %We propose to over-approximate all possible realizations of $\phi^{(i)}$ in the original formulation by transforming the problem:
\begin{align}\label{eq:upperbound}
    \overline{\epsilon} = \maxi_{\varphi^{(1..n)}} \quad \dfrac{1}{n}\sum_{i=1}^{n} &\overline{\mathcal{I}(f^{\theta}, x, \delta + \varphi^{(i)} )} \quad s.t. \quad   \varphi^{(i)} \in \mathbb{R}^{+},\ \dfrac{1}{n}\sum_{i=1}^{n} || \varphi^{(i)} ||^p \leq \gamma^p
\end{align}
The key change to the formulation is that we are no longer considering input transformations parameterized by $\phi^{(i)}$, but instead consider an over-approximation of all possible $\phi^{(i)}$ that are within a $\varphi^{(i)}$-radius around the point $x^{(i)}$. Moreover, it is clear that picking any assignment of the $\varphi^{(i)}$ values will \textit{not} constitute an upper-bound on the DIF violation. However, the global maximizing assignment of $\varphi^{(i)}$ values \textit{is} guaranteed to over-approximate the value of Equation~\eqref{eq:optimproblem}, formally stated in Theorem \ref{thm:upperboundthm}. 
\begin{theorem}\label{thm:upperboundthm}
  Given an optimal assignment of $\{\varphi^{(i)} \}_{i=1}^n$ in Equation~\eqref{eq:upperbound}, the corresponding $\overline{\epsilon}$ is a sound upper-bound on the DIF violation of the model and therefore, is a certificate that no $\gamma$-Wasserstien distribution shift can cause the individual fairness of the model to exceed $\overline{\epsilon}$. 
\end{theorem}
The proof of Theorem \ref{thm:upperboundthm} is contained in Appendix~\ref{appendix:upperboundproof}. Given that the function $\overline{\mathcal{I}(f^{\theta}, x, \delta + \varphi^{(i)} )}$ is a continuous and monotonically increasing function, we have that the maximization problem in Equation~\eqref{eq:upperbound} is quasi-convex. After proving this function is also H\"{o}lder continuous, one can guarantee convergence in finite time to the optimal solution with bounded error using recent advancements in quasi-convex optimization \citep{hu2020convergence, agrawal2020disciplined, hu2022quasi}. As the global optimal solution is guaranteed to be an upper-bound, we can use the methods for quasi-convex optimization to compute a certificate for the DIF property. In the Appendix we prove the H\"{o}lder continuity of $\mathcal{I}$ and conduct numerical convergence experiments to validate this theory. Given that the result of this optimization is a certified upper bound on the DIF value, we denote its result with $\overline{I_\gamma(f^{\theta}, X, \delta)}$ where $\gamma$ is the distributional parameter of the DIF formulation. While $\overline{\epsilon}$ is a valid upper-bound on the DIF violation for the given dataset, in the low-data regime one may have non-negligible error due to poor approximation of the data distribution. The finite sample approximation error can be bounded by an application of Hoeffding's inequality as we state in Lemma \ref{lem:finitesample}
\begin{lemma}\label{lem:finitesample}
  Given an upper-bound, $\overline{\epsilon}$ (a global maximizing assignment to Equation~\eqref{eq:upperbound}) we can bound the error induced by using a finite sample from above by $\tau$. That is, the estimate $\overline{\epsilon}$ is within $\tau$ of the true expectation of $\overline{\epsilon}$ with probability $1-\lambda$ as long as $n$ is at least $(-1/2\tau^2)\log (\lambda/2)$
\end{lemma}
\subsection{Training Distributionally Individually Fair NNs}

It is well-known that guarantees for certification agnostic neural networks can be vacuous \citep{gowal2018effectiveness, wicker2021bayesian}. To combat this, strong empirical results have been obtained for neural networks trained with certified regularizers \citep{benussi2022individual, khedr2022certifair}. In this work, we have proposed three novel, differentiable bounds on certified individual fairness. We start by proposing a loss on certified local individual fairness similar to what is proposed in \citet{benussi2022individual} but substituted with our significantly more efficient interval bound procedure:
$$ \mathcal{L}_{\text{F-IBP}} = \mathcal{L}(f^\theta, X, Y) + \alpha \overline{\mathcal{I}_{0}(f^\theta, X, \delta)}$$
where $\alpha$ is a weighting that trades off the fair loss with the standard loss. When $\gamma = 0$, we recover exactly a local constraint on individual fairness for each $x^{(i)} \in X$. % Thus, this regularization is exactly what is proposed in \citet{benussi2022individual} but using interval bound propagation for training as in \citet{gowal2018effectiveness}. 
By increasing the $\gamma$ from zero we move from a local individual fairness constraint to a DIF constraint:
$$ \mathcal{L}_{\text{U-DIF}} = \mathcal{L}(f^\theta, X, Y) + \alpha \overline{\mathcal{I}_{\gamma}(f^\theta, X, \delta)}$$
where U-DIF stands for upper-bound on DIF. This loss comes at the cost of increased computational time, but is a considerably stronger regularizer. In fact, given that the upper bound on certified DIF might be vacuous, and therefore difficult to optimize, at the start of training, we also propose a training loss that optimizes the lower bound on DIF as it may be empirically easier to optimize and serves as a middle ground between F-IBP and U-DIF:
$$ \mathcal{L}_{\text{L-DIF}} = \mathcal{L}(f^\theta, X, Y) + \alpha \text{$\underline{\mathcal{I}_{\gamma}(f^\theta, X, \delta)}$}$$

\section{Experiments}

%In this section, we study the effect of our proposed training methods and empirically validate our certification method for distributional individual fairness. We start here by describing the datasets and metrics used in our evaluation. After which we provide a comparative analysis of five different individual fairness training methods. We then study the scalability of our proposed training method and the tightness of of our upper and lower bounds. Finally, we use real-world distribution shifts stemming from changes over time and to geographic context. We use these distribution shifts in order to study the potential real-world impact of our methods.

In this section we empirically validate our proposed method on a variety of datasets. We first describe the datasets and metrics used. We then cover a wide range of experimental ablations and validate our bounds on real-world distribution shifts. We conclude the section with a study of how IF training impacts other notions of fairness.\footnote{Code to reproduce experiments can be found at: \url{https://github.com/matthewwicker/DistributionalIndividualFairness}}

\textbf{Datasets} We benchmark against the Adult, Credit, and German datasets from the UCI dataset repository \citep{uci}. The German or Satlog dataset predicts credit risk of individuals. The Credit dataset predicts if an individual will default on their credit. The Adult dataset predicts if an individuals income is greater than 50 thousand dollars. We additionally consider three datasets from the Folktables datasets, Income, Employ and Coverage, which are made up of millions of data points curated from the 2015 to 2021 US census data \citep{ding2021retiring}. The Income dataset predicts if an individual makes more than 50 thousand US dollars, the Employ dataset predicts if an individual is employed, and the Coverage dataset predicts if an individual has health insurance. For each dataset gender is considered the protected attribute.

\textbf{Metrics} We consider four key metrics. For performance, we measure accuracy which is computed over the entire test-set. Next we consider the local fairness certificates (LFC), which is taken as $1/n \sum_{i=1}^{n} \overline{\mathcal{I}_{0}(f^{\theta}, x_i, \delta)}$. Next we consider the empirical distributional fairness certificate (E-DFC) which is the maximum of the LFC taken over a set of observed distribution shifts. That is, we observe a set of individuals drawn from a distribution shift, $x^{(k)} \sim P^{k}$, where $P^k$ represents a shift of geographic context (e.g., $P^0$ being individuals from California and $P^1$ being individuals from North Carolina) or a shift in time. % (e.g., $P^0$ being individuals observed in 2015 and $P^1$ being individuals observed in 2016). 
E-LFC is then defined as $\max_{j \in [k]} 1/n \sum_{i=0}^{n} \overline{\mathcal{I}_{0}(f^{\theta}, x^{(j)}_i, \delta)}$. Finally, we consider the adversarial distributional fairness certification (A-DFC) which is our computed value of $\overline{\mathcal{I}_{\gamma}(f^{\theta}, x_i, \delta)}$. Unless stated otherwise, we use $\delta = 0.05$, $\gamma = 0.1$, and with respect to 1000 test-set individuals. Complete experimental details are given in Appendix~\ref{appendix:experimentaldetails}.

\begin{table}[]\hspace{-0.00em}\footnotesize\addtolength{\tabcolsep}{-3.85pt}\centering
\begin{tabular}{l|llll|l|llll|l|llll|}
\cline{2-5} \cline{7-10} \cline{12-15}
                             & \multicolumn{4}{c|}{Income}                                                                                                      &  & \multicolumn{4}{c|}{Employ}                                                                                                      &  & \multicolumn{4}{c|}{Coverage}                                                                                                   \\ \cline{2-5} \cline{7-10} \cline{12-15} 
                             & \multicolumn{1}{l|}{Acc}            & \multicolumn{1}{l|}{LFC}            & \multicolumn{1}{l|}{E-DFC}          & A-DFC          &  & \multicolumn{1}{l|}{Acc}            & \multicolumn{1}{l|}{LFC}            & \multicolumn{1}{l|}{E-DFC}          & A-DFC          &  & \multicolumn{1}{l|}{Acc}            & \multicolumn{1}{l|}{LFC}            & \multicolumn{1}{l|}{E-DFC}         & A-DFC          \\ \cline{1-5} \cline{7-10} \cline{12-15} 
\multicolumn{1}{|l|}{FTU}    & \multicolumn{1}{l|}{\textbf{0.820}} & \multicolumn{1}{l|}{0.999}          & \multicolumn{1}{l|}{0.999}          & 1.000          &  & \multicolumn{1}{l|}{\textbf{0.809}} & \multicolumn{1}{l|}{0.891}          & \multicolumn{1}{l|}{0.906}          & 0.999          &  & \multicolumn{1}{l|}{\textbf{0.721}} & \multicolumn{1}{l|}{1.000}          & \multicolumn{1}{l|}{1.000}         & 1.000          \\ \cline{1-5} \cline{7-10} \cline{12-15} 
\multicolumn{1}{|l|}{SenSR}  & \multicolumn{1}{l|}{0.782}          & \multicolumn{1}{l|}{0.895}          & \multicolumn{1}{l|}{0.931}          & 0.995          &  & \multicolumn{1}{l|}{0.743}          & \multicolumn{1}{l|}{0.247}          & \multicolumn{1}{l|}{0.300}          & 0.639          &  & \multicolumn{1}{l|}{0.709}          & \multicolumn{1}{l|}{0.538}          & \multicolumn{1}{l|}{0.865}         & 0.899          \\ \cline{1-5} \cline{7-10} \cline{12-15} 
\multicolumn{1}{|l|}{F-IBP}  & \multicolumn{1}{l|}{0.771}          & \multicolumn{1}{l|}{0.076}          & \multicolumn{1}{l|}{0.092}          & 0.130          &  & \multicolumn{1}{l|}{0.743}          & \multicolumn{1}{l|}{0.040}          & \multicolumn{1}{l|}{0.090}          & 0.178          &  & \multicolumn{1}{l|}{0.699}          & \multicolumn{1}{l|}{0.162}          & \multicolumn{1}{l|}{0.253}         & 0.288          \\ \cline{1-5} \cline{7-10} \cline{12-15} 
\multicolumn{1}{|l|}{L-DIF} & \multicolumn{1}{l|}{0.763}          & \multicolumn{1}{l|}{0.035}          & \multicolumn{1}{l|}{0.051}          & 0.095          &  & \multicolumn{1}{l|}{0.738}          & \multicolumn{1}{l|}{0.025}          & \multicolumn{1}{l|}{0.051}          & 0.091          &  & \multicolumn{1}{l|}{0.698}          & \multicolumn{1}{l|}{0.101}          & \multicolumn{1}{l|}{0.128}         & 0.136          \\ \cline{1-5} \cline{7-10} \cline{12-15} 
\multicolumn{1}{|l|}{U-DIF} & \multicolumn{1}{l|}{0.725}          & \multicolumn{1}{l|}{\textbf{0.002}} & \multicolumn{1}{l|}{\textbf{0.002}} & \textbf{0.042} &  & \multicolumn{1}{l|}{0.724}          & \multicolumn{1}{l|}{\textbf{0.000}} & \multicolumn{1}{l|}{\textbf{0.006}} & \textbf{0.067} &  & \multicolumn{1}{l|}{0.681}          & \multicolumn{1}{l|}{\textbf{0.000}} & \multicolumn{1}{l|}{\textbf{0.042}} & \textbf{0.071} \\ \cline{1-5} \cline{7-10} \cline{12-15} 
\end{tabular}
\vspace{1.5em}
\caption{Performance of each training method across three folktables datasets. For each dataset we give the accuracy (Acc, $\uparrow$), local IF violation (LFC, $ \downarrow$), empirical DIF violation (E-DFC, $\downarrow$), and the adversarial DIF violation (A-DFC, $\downarrow$). We observe a consistent drop in accuracy across datasets as we enforce more strict DIF constraints. However, we notice orders of magnitude decrease in the individual fairness violation across all three metrics as methods impose stricter constraints. }\label{tab:trainingcomp}
\end{table}

\subsection{Comparing Training Methods}

In Table~\ref{tab:trainingcomp}, we compare different individual fairness training methods on three datasets from \cite{ding2021retiring}. For each dataset, we train a two layer neural network with 256 hidden neurons per layer. We compare with fairness through unawareness (FTU) which simply removes the sensitive feature and sensitive subspace robustness (SenSR) \citep{yurochkin2019training}. %We hold all hyper-parameters for the networks, e.g., activation function and optimizer constant for each method to ensure a fair comparison. 
We observe a consistent decrease in accuracy as we increase the strictness of our constraint. While FTU has the best performance, we are unable to compute strong IF guarantees. This result is expected given that it is well-known that guaranteeing the performance of certification-agnostic networks is empirically challenging \citep{gowal2018effectiveness}. Similarly, we observe that SenSR has a positive effect on the DIF guarantee while paying a 2-5\% accuracy cost. When only enforcing local guarantees with F-IBP, we observe a 3-5\% decrease in accuracy, but achieve an order of magnitude decrease in the local and distributional fairness violations. Using both L-DIF and U-DIF results in the strongest IF and DIF guarantees, albeit at a larger accuracy cost. Each regularization method proposed provides a different fairness-accuracy trade-off which should be carefully chosen to suit the needs of a given application. %Further discussion of the performance w.r.t other fairness metrics is considered in Section~\ref{sec:performancetradeoff}.

\begin{figure}
    \centering
\includegraphics[width=0.9\textwidth]{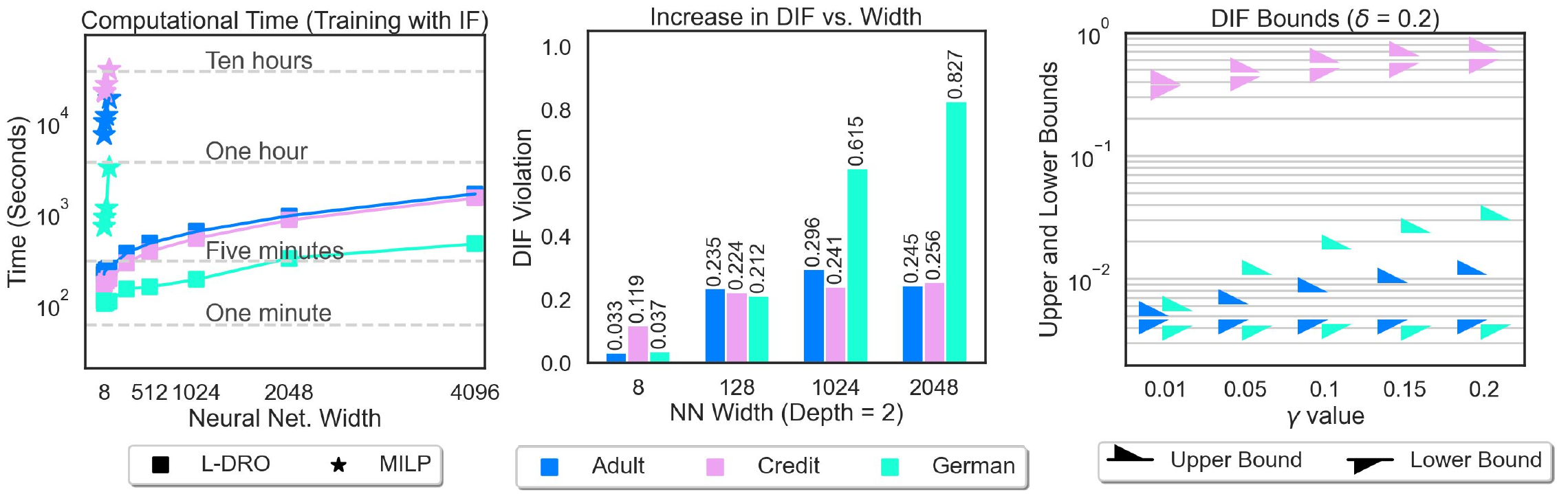}
    \caption{Empirical algorithm analysis for our upper and lower bounds on DIF. \textbf{Left:} Computational time comparison between MILP training (times reported in \citep{benussi2022individual}) and L-DIF demonstrates our methods considerable advantage in scalability. \textbf{Center:} As the NNs get larger our DIF certificates get looser, as expected with bound propagation methods \citep{gowal2018effectiveness}. \textbf{Right:} As we increase $\gamma$ we also increase the gap between our upper and lower bounds due to the over-approximate nature of our upper-bounds.}
    \label{fig:boundanaly}
\end{figure}

\subsection{Empirical Bound Analysis}

In this section, we provide an overview of our algorithm's performance as we vary architecture size and the distributional parameter $\gamma$. 
In the left-most plot of Figure~\ref{fig:boundanaly}, we use stars to plot the computational times reported to train a single hidden layer neural network from \citep{benussi2022individual}. We use squares to denote the amount of time used by our most computationally demanding method (L-DIF). Our method scales up to 4096 neurons without crossing the one hour threshold. We report an extended version of this experiment in Appendix~\ref{appendix:experiments}. Moreover, the results from this paper were run on a laptop while the timings from \citep{benussi2022individual} use a multi-GPU machine.

In the center plot of Figure~\ref{fig:boundanaly}, we show how the value of our upper bound on the DIF violation grows as we increase the size of the neural network. As the neural network grows in size, we observe a steady increase in our upper-bounds. This is not necessarily because the learned classifier is less individually fair. The steady increase in the bound value can be attributed to the approximation incurred by using bound propagation methods. It is well-known that as the neural networks grow larger, certification methods become more approximate \citep{gehr2018ai2, wicker2021bayesian}. 

In the right-most plot of Figure~\ref{fig:boundanaly}, we show how our upper and lower bounds respond to increasing the distributional parameter $\gamma$. As $\gamma$ grows, we expect that our bounds become loose due to the over-approximation inherent in our bounds. Here we consider a two-layer 16 neuron per-layer NN.  Indeed we observe in the right-most plot of Figure~\ref{fig:boundanaly}, that the gap between our upper and lower bounds grow to almost an order of magnitude when we increase $\gamma$ from 0.01 to 0.2. %In future, the tightness of our bounds could be improved with tighter certification techniques to compute $\overline{\mathcal{I}_{\gamma}(f^\theta, X, \delta)}$.% We hope that future work will sharpen the initial bounds provided by this paper. [Make a statement about how the bound gap grows really large for German which is only represented by 800 individuals in the training set]

\begin{figure}
    \centering
\includegraphics[width=1.0\textwidth]{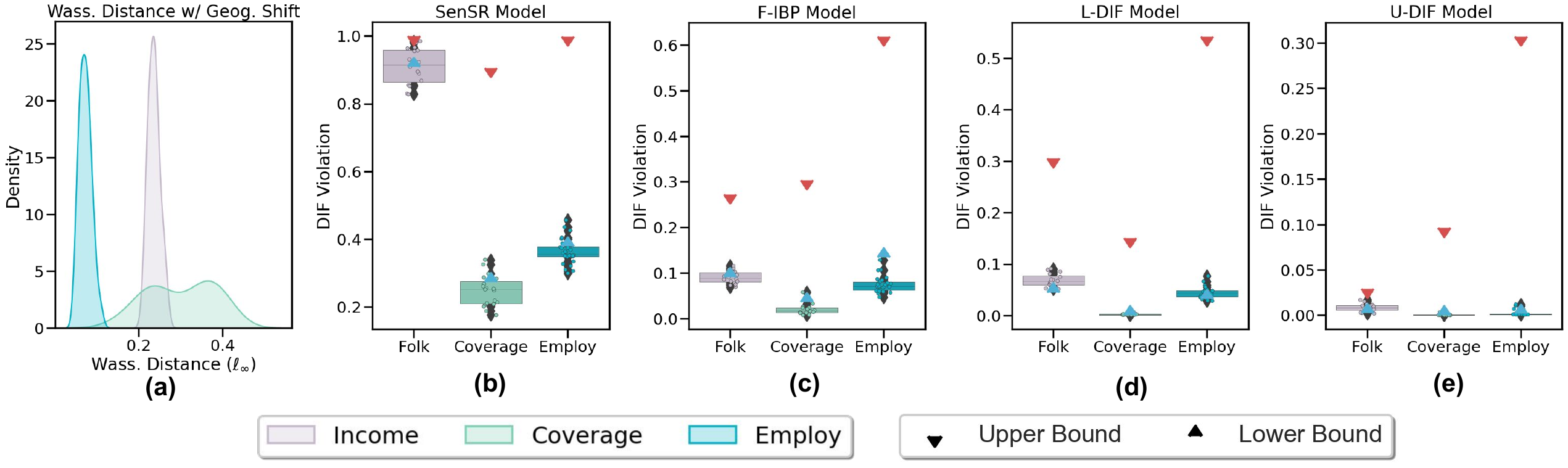}
    \vspace{-1em}
    \caption{Evaluating our bounds versus real-world distribution shifts. \textbf{Column (a):} An empirical distribution of Wasserstein distances between the distribution of individuals from different pairs of states, we certify w.r.t the upper quartile of these distributions. \textbf{Columns (b) - (e):} We plot the local fairncess certificates (LFC) for each of the shifted dataset using a boxplot. We then plot our lower bound on the worst-case DIF violation as a blue triangle and our upper bound on the worst-case DIF violation as a red triangle.}
    \label{fig:realworldshift}
\end{figure}

\subsection{Certification for Real-World Distribution Shifts}

In this section we compare our adversarial bounds on DIF against real-world distribution shifts. In Figure~\ref{fig:realworldshift} we study our ability to bound a shift in geographical context. We use the folktables dataset \citep{ding2021retiring} and assume we can only train our model using data from California. We would then like to ensure that a our model is individually fair when applied to other states. We start by computing the Wasserstein distance between the data observed for different states. The distribution of these values is plotted in  Figure~\ref{fig:realworldshift} (a). By estimating the pairwise Wasserstein distance between the states data we can estimate what $\gamma$. In practice, one would use domain knowledge to estimate $\gamma$. Each dataset has a different worst-case $\gamma$ value, i.e., 0.15 for Employ versus 0.4 for Coverage. In Figure~\ref{fig:realworldshift} (b)-(e) we then compute our upper bounds (red triangle) and lower bounds (blue triangle) for each observed value of $\gamma$. We also plot the observed local individual fairness certificates centered at each of the shifted datasets, i.e., the local fairness of our model when applied to North Carolina etc., and plot this as a box plot. We observe that our lower-bound tightly tracks the worst-case observed individual fairness. Our upper-bound is over-approximate but reasonably tight, especially for NNs trained with DIF guarantees. The validation of our upper and lower bounds against real world DIF violations highlights the value of our bounds as a scalable, practical, and sound source of guarantees for individual fairness such that model stakeholders can justify deployment.

\section{Conclusion}

In conclusion, our proposed method addresses the crucial need for scalable, formal guarantees of individual fairness. Our novel convex approximation of IF constraints enables efficient real-time audits of individual fairness by significantly reducing the computational cost local IF certification. We introduced the first certified bounds on DIF and demonstrated their applicability to significantly larger neural networks than previous works. Our empirical investigation of real-world distribution shifts further validated the scalability, practicality, and reliability of the proposed bounds. Our bounds on DIF violations offer a robust source of guarantees for individual fairness, enabling model stakeholders to justify the deployment of fair machine learning algorithms.

\textbf{Limitations} This work improves the scalability of prior IF guarantees by more than two orders of magnitude. However, the models we consider are smaller than models that may be deployed in practice we hope to further improve scalability in future work. %In addition, this work applies solely on classification and regression settings, in future work we hope to see DIF adapted to a richer class of problems.

\textbf{Broader Impact} This work focuses on fairness quantification and bias reduction which is one critical aspect of developing socially responsible and trustworthy machine learning (ML). We highlight that this work focuses on individual fairness which is only one aspect of fairness in machine learning and on its own does not constitute a complete, thorough fairness evaluation. Experimental analysis of the impact of DIF on group fairness is in Appendix \ref{sec:performancetradeoff}. The intention of this work is to improve the state of algorithmic fairness in ML, and we do not foresee negative societal impacts.

\bibliography{bib}
\bibliographystyle{bibstyle} 

\clearpage
\newpage
\appendix

\section{Experimental Details}\label{appendix:experimentaldetails}

\begin{table}[]\footnotesize\addtolength{\tabcolsep}{-3.85pt}\centering
\begin{tabular}{|l|c|c|c|c|c|c|c|}
\hline
\multicolumn{1}{|c|}{Dataset} & \multicolumn{1}{l|}{Depth} & \multicolumn{1}{l|}{Width} & \multicolumn{1}{l|}{Learning Rate} & \multicolumn{1}{l|}{Epochs} & \multicolumn{1}{l|}{Delta} & \multicolumn{1}{l|}{Gamma} & \multicolumn{1}{l|}{Optimizer} \\ \hline
Adult                         & 2                          & 256                        & 0.001                              & 50                          & 0.02                       & 0.025                      & Adam (momentum=0.9)            \\ \hline
Credit                        & 2                          & 256                        & 0.01                               & 50                          & 0.02                       & 0.025                      & Adam (momentum=0.9)            \\ \hline
German                        & 2                          & 256                        & 0.0025                             & 50                          & 0.02                       & 0.025                      & Adam (momentum=0.9)            \\ \hline
Income                        & 2                          & 256                        & 0.001                              & 50                          & 0.02                       & 0.025                      & Adam (momentum=0.9)            \\ \hline
Coverage                      & 2                          & 256                        & 0.001                              & 50                          & 0.02                       & 0.025                      & Adam (momentum=0.9)            \\ \hline
Employ                        & 2                          & 256                        & 0.001                              & 50                          & 0.02                       & 0.025                      & Adam (momentum=0.9)            \\ \hline
\end{tabular}\vspace{1em}
\caption{Hyperparameters for the base model of each dataset. We keep all parameters across models the same with the exception of learning rate which is fine-tuned according to best accuracy on a validation set.}\label{tab:hyperparams}
\end{table}

In this section, we report the hyperparameters of each base model used in our paper, details in Table~\ref{tab:hyperparams}. The only hyperparameter that is tuned is done per dataset using a 10\% validation split. The best learning rate is then chosen after a grid search over 20 evenly spaced values. We have chosen a model architecture that is considerably larger than what is used by previous certification works in order to understand if the claimed scalability benefits of our method holds in practice over a half dozen datasets and compared to real-world distribution shifts. Indeed, we find in the main text that despite our base model being considerably larger than previous works, we are able to get strong IF and DIF guarantees despite the models size. In our further experiments, we keep all hyperparameters constant unless otherwise noted. For example, in our exploration of increased width we only vary the width holding all other hyperparameters in Table~\ref{tab:hyperparams} constant.

\section{Additional Experimental Ablations}\label{appendix:experiments}

In this section, we report further experimental abalations that validate the effectiveness of our proposed method. We start by studying the empirical convergence of our approach to the true global optimal value which is computed numerically. We then provide an extended analysis of a real world distribution shift which stems from population shift over time (specifically between 2015 and 2022). Finally, we report on the scalability of our proposed F-IBP training method where we find that it has remarkable scalability benefits. 

\begin{figure}[h]
    \centering
\includegraphics[width=1.0\textwidth]{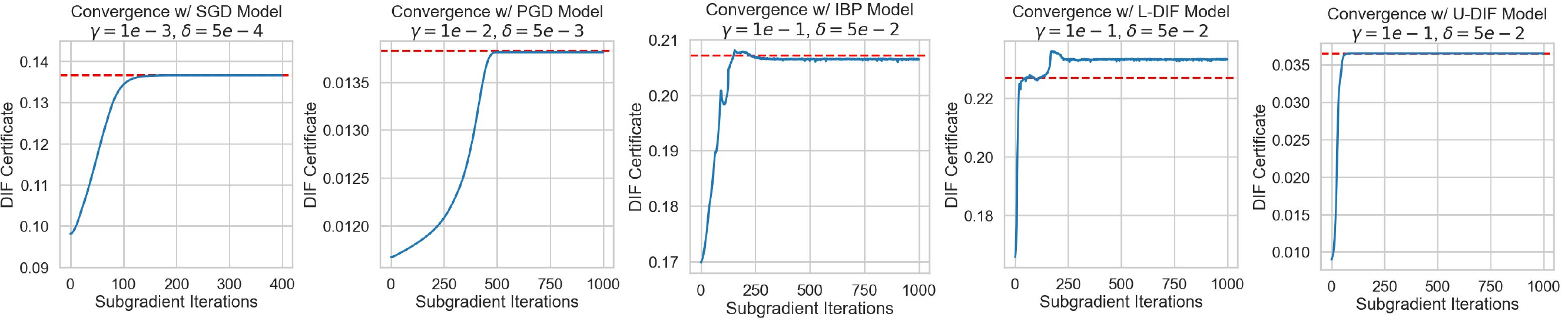}
    \caption{Convergence analysis for subgradient methods solution to Equation~\eqref{eq:upperbound}. Blue line plots the value of the optimization problem over a fixed number of iterations while the red dashed line indicates the numerically optimal solution to the optimization problem. We consider the Income dataset on networks trained with various levels of DIF regularization. From left to right we run our convergence analysis on neural networks trained with SGD, SenSR, F-IBP, L-DIF, and U-DIF regularization.  }
    \label{fig:convergence}
\end{figure}

\begin{table}[H]\centering
\begin{tabular}{c|ccc|}
\cline{2-4}
                            & \multicolumn{3}{c|}{Numerical Error}                                     \\ \cline{2-4} 
                            & \multicolumn{1}{c|}{Income}  & \multicolumn{1}{c|}{Employ}   & Coverage  \\ \hline
\multicolumn{1}{|c|}{SGD}   & \multicolumn{1}{c|}{-0.0143} & \multicolumn{1}{c|}{0.03706}  & 0.113219  \\ \hline
\multicolumn{1}{|c|}{SenSR} & \multicolumn{1}{c|}{0.06885} & \multicolumn{1}{c|}{0.03659}  & -0.033412 \\ \hline
\multicolumn{1}{|c|}{F-IBP} & \multicolumn{1}{c|}{0.00820} & \multicolumn{1}{c|}{-0.00098} & 0.022244  \\ \hline
\multicolumn{1}{|c|}{L-DRO} & \multicolumn{1}{c|}{0.05840} & \multicolumn{1}{c|}{-0.00132} & 0.018893  \\ \hline
\multicolumn{1}{|c|}{U-DRO} & \multicolumn{1}{c|}{-0.0021} & \multicolumn{1}{c|}{0.008271} & -0.007846 \\ \hline
\end{tabular}\vspace{1em}\caption{For each combination of regularization method and Folktables dataset we compare the error between a subgradient solution to Equation~\eqref{eq:upperbound} and the numerically optimal value. Positive entries indicates that the solution found by the subgradient method was greater than that of the numerical algorithm and visa versa for negative entries. We highlight that on average the actual value of the optimization was roughly 0.3, therefore error on the order of 1e-2 is acceptable.}\label{tab:numericalerror}
\end{table}

\subsection{Convergence Experiments}

In Section~\ref{appendix:proofs} we provide proofs that solving for a global optimum of Equation~\eqref{eq:upperbound} is a true upper bound to the DIF violation and that any suitable subgradient algorithms should converge to a global optimum in finite time. In this Section, we discuss the experimental convergence of our U-DIF algorithm to the global optimum.  In order to approximately compute the true global optimum, we use the following numerical scheme. We assume we would like to compute the DIF violation according to $\delta = 0.1$ and $\gamma = 0.05$ (exact numbers vary by network and are given in Figure~\ref{fig:convergence}). We further assume that we are given 50 individuals to compute the DIF i.e., $n=50$. We start by evaluating the function $\overline{\mathcal{I}(f^{\theta}, x^{(i)}, \delta + \varphi^{(i,k)})}$ for each individual and for each $k \in [K]$ (we choose $k=500$) values of $\varphi$ evenly spaced between 0.0 and 2.4 as this (more generally $(n\gamma) - \delta$) is an upper bound to the radius that any one individual can have without breaking the Wasserstein constraint. After evaluating $\overline{\mathcal{I}(f^{\theta}, x^{(i)}, \delta + \varphi^{i,k})}$ for each of the $k$-many $\varphi^{(i,k)}$ values for all individuals, we exhaustively enumerate all possible assignments of $\{\varphi^{(i,k)}\}_{i=1}^{50}$ such that $1/50 \sum_{i=1}^{50} \varphi^{(i)}$ is less than or equal to 0.05 ($\gamma$). As $K \rightarrow \infty$, this produces an exact computation of the global maximum of Equation~\eqref{eq:upperbound}.

Once a numerical solution has been computed, we run our U-DIF optimization using the method of \cite{hu2020convergence}  and we plot the results for the Income dataset in Figure~\ref{fig:convergence}. We find that regardless of how the neural network is trained, our algorithm quickly converges to the global optimum assignment of $\{\varphi^{(i)}\}_{i=1}^{n}$. We note that in the center plot of Figure~\ref{fig:convergence}, corresponding to the F-IBP trained network, we converge to a value slightly below the global optimum; however, this is within the expected convergence error. In order to better understand this expected convergence error, we run the experiment from Figure~\ref{fig:convergence} across each of the folktables datasets, and report the numerical error of our U-DIF computed value in Table~\ref{tab:numericalerror}. Only in cases where the value reported is negative do we converge to something smaller than the value produced by the numerical solution, and in each case the numerical error of the subgradient procedure converges to within tolerable error as theoretically expected.

\subsection{Extended results on real-world shifts}

In this Section we briefly describe Figure~\ref{fig:timeshift} which presents a complementary empirical study to that of Figure~\ref{fig:realworldshift} in the main text. In the far left hand side of Figure~\ref{fig:timeshift} we plot the Wasserstein distance between pairs of datasets that differ only in the year the data was collected. For the Folktables datasets this varies between 2015 and 2021. We find that the data shift over time is less severe than the data shift over geographic location. Again we use the upper quartile of these distributions as the value that we would like to certify. In Figure~\ref{fig:timeshift} (b)-(e) we plot the distribution of IF violation for an observed temporal shift as a box plot. For each dataset, we observe that more strict DIF regularization corresponds to better performance with respect to empirical DIF violations. Moreover, in Figure~\ref{fig:timeshift} (b)-(e) we plot our computed upper and lower DIF bounds. We find, as in the main text, that our lower-bound tightly tracks the empirically observed worst-case DIF violation, while our upper bound is a sound over-estimate of the DIF violation.

\subsection{Extended scalability results}

In Figure~\ref{fig:fibptimings} we report the computational time required by our proposed F-IBP method as well as the computational times reported by \citet{benussi2022individual}. We highlight that the times reported by \citet{benussi2022individual} utilize a multi-GPU machine while our experiments are conducted using only a consumer-grade laptop. Despite a considerable compute disadvantage, our methods never takes longer than 5 minutes to train neural networks with more than 4000 neurons whereas the training required by the MILP procedure can take up to 10 hours to train a 64 neuron NN. This impressive scalability improvement underscores the value of our proposed method.

\begin{table}[]\hspace{-0.0em}\footnotesize\addtolength{\tabcolsep}{-3.85pt}\centering
\begin{tabular}{l|llll|l|llll|l|llll|}
\cline{2-5} \cline{7-10} \cline{12-15}
                            & \multicolumn{4}{c|}{Income}                                                                                                      &  & \multicolumn{4}{c|}{Employ}                                                                                                       &  & \multicolumn{4}{c|}{Coverage}                                                                                                     \\ \cline{2-5} \cline{7-10} \cline{12-15} 
                            & \multicolumn{1}{l|}{D. Par.}        & \multicolumn{1}{l|}{Eq. Od}        & \multicolumn{1}{l|}{Eq. Op}        & IF. Par       &  & \multicolumn{1}{l|}{D. Par}        & \multicolumn{1}{l|}{Eq. Od}        & \multicolumn{1}{l|}{Eq. Op}         & IF. Par        &  & \multicolumn{1}{l|}{D. Par}        & \multicolumn{1}{l|}{Eq. Od}        & \multicolumn{1}{l|}{Eq. Op}         & IF. Par       \\ \cline{1-5} \cline{7-10} \cline{12-15} 
\multicolumn{1}{|l|}{SGD}   & \multicolumn{1}{l|}{0.074}          & \multicolumn{1}{l|}{\textbf{0.027}} & \multicolumn{1}{l|}{0.006}          & \textbf{0.000} &  & \multicolumn{1}{l|}{0.012}          & \multicolumn{1}{l|}{0.103}          & \multicolumn{1}{l|}{0.003}           & 0.011          &  & \multicolumn{1}{l|}{\textbf{0.006}} & \multicolumn{1}{l|}{\textbf{0.023}} & \multicolumn{1}{l|}{\textbf{0.0082}} & 0.012          \\ \cline{1-5} \cline{7-10} \cline{12-15} 
\multicolumn{1}{|l|}{SenSR} & \multicolumn{1}{l|}{0.009}          & \multicolumn{1}{l|}{0.139}          & \multicolumn{1}{l|}{0.037}          & 0.021          &  & \multicolumn{1}{l|}{0.049}          & \multicolumn{1}{l|}{0.156}          & \multicolumn{1}{l|}{0.017}           & 0.041          &  & \multicolumn{1}{l|}{0.055}          & \multicolumn{1}{l|}{0.112}          & \multicolumn{1}{l|}{0.074}           & 0.009          \\ \cline{1-5} \cline{7-10} \cline{12-15} 
\multicolumn{1}{|l|}{F-IBP} & \multicolumn{1}{l|}{0.013}          & \multicolumn{1}{l|}{0.081}          & \multicolumn{1}{l|}{0.014}          & \textbf{0.000} &  & \multicolumn{1}{l|}{0.030}          & \multicolumn{1}{l|}{0.119}          & \multicolumn{1}{l|}{\textbf{0.0009}} & 0.042          &  & \multicolumn{1}{l|}{0.033}          & \multicolumn{1}{l|}{0.070}          & \multicolumn{1}{l|}{0.055}           & 0.008          \\ \cline{1-5} \cline{7-10} \cline{12-15} 
\multicolumn{1}{|l|}{L-DIF} & \multicolumn{1}{l|}{0.021}          & \multicolumn{1}{l|}{0.157}          & \multicolumn{1}{l|}{0.020}          & \textbf{0.000} &  & \multicolumn{1}{l|}{0.016}          & \multicolumn{1}{l|}{0.092}          & \multicolumn{1}{l|}{0.001}           & 0.042          &  & \multicolumn{1}{l|}{0.032}          & \multicolumn{1}{l|}{0.069}          & \multicolumn{1}{l|}{0.054}           & 0.006          \\ \cline{1-5} \cline{7-10} \cline{12-15} 
\multicolumn{1}{|l|}{U-DIF} & \multicolumn{1}{l|}{\textbf{0.000}} & \multicolumn{1}{l|}{0.094}          & \multicolumn{1}{l|}{\textbf{0.004}} & \textbf{0.000} &  & \multicolumn{1}{l|}{\textbf{0.009}} & \multicolumn{1}{l|}{\textbf{0.082}} & \multicolumn{1}{l|}{0.002}           & \textbf{0.000} &  & \multicolumn{1}{l|}{0.028}          & \multicolumn{1}{l|}{0.064}          & \multicolumn{1}{l|}{0.054}           & \textbf{0.000} \\ \cline{1-5} \cline{7-10} \cline{12-15} 
\end{tabular}
\vspace{0.35em}
\caption{Affect of IF training on measures of group fairness. For each metric lower is better. We specifically provide computations of demographic parity (D. Par), equalized odds (Eq. Od), equalized opportunity (Eq. Op), and IF parity (IF. Par). More often than not, we find that DIF training improves measures group fairness; though in some instances makes group fairness considerably worse.  }\label{tab:groupfair}
\end{table}

\subsection{Impact on Group Fairness}\label{sec:performancetradeoff}
While individual fairness is a flexible and key measure of fairness, it is currently the case that no one fairness metric alone captures a complete picture of model bias \citep{fazelpour2020algorithmic}. In this section, we briefly report on the effect of our proposed training method on group fairness notions. In Table~\ref{tab:groupfair}, we report common group fairness notions such as demographic parity, equalized odds, and equalized opportunity for each of our trained models. Additionally we consider IF parity which is taken as the difference in average local individual fairness violation of the model w.r.t the majority group (men) and the minority group (women). Across these metrics, we find that more often than not DIF training also has a positive effect on group fairness. While this is not always the case, Table~\ref{tab:groupfair} clearly establishes that DIF training methods do not necessarily exacerbate other forms of model bias.

\section{Extended Related Works}\label{appendix:relatedworks}

In the main text, Section~\ref{sec:relatedworks} describes the context of our contribution relative to the literature on guarantees for individual fairness (local and global) as well as works related to distributionally robust fairness. In this section, we describe some works in distributional robustness that have a strong relationship to our contribution. In \cite{sinha2017certifying} the authors provide an initial certificate of distributional robustness under many of the same assumptions as our method e.g., restricting to Wasserstein distances between Dirac measures. However, their bound is derived based on a Taylor expansion and relies upon a knowing the value of global Lipschitz constants of the neural network. Our method, on the other hand, makes no such assumptions on the value of Lipschitz constants. Indeed, our bound relies on the H\"{o}lder continuity of the $\mathcal{I}$ function; however, we do not rely on any knowledge of value of the constants for which the  H\"{o}lder condition holds. The requirement of the H\"{o}lder condition is a practical consideration in order to ensure that our bounds converge globally with subgradient methods. Using branch and bound would solve our optimization problem without any continuity assumption on $\mathcal{I}$, albeit at greater computational complexity. For further works focusing on distributional robustness when using the Wasserstien distance, we point interested readers to \cite{blanchet2016quantifying} and \cite{mohajerin2018data}. Since the submission of this work, \citep{doherty2023individual} has investigated individual fairness in the context of Bayesian neural networks and find that uncertainty benefits individual fairness, perhaps due to its relationship to adversarial robustness \citep{carbone2020robustness}.

\begin{figure}
    \centering
\includegraphics[width=1.0\textwidth]{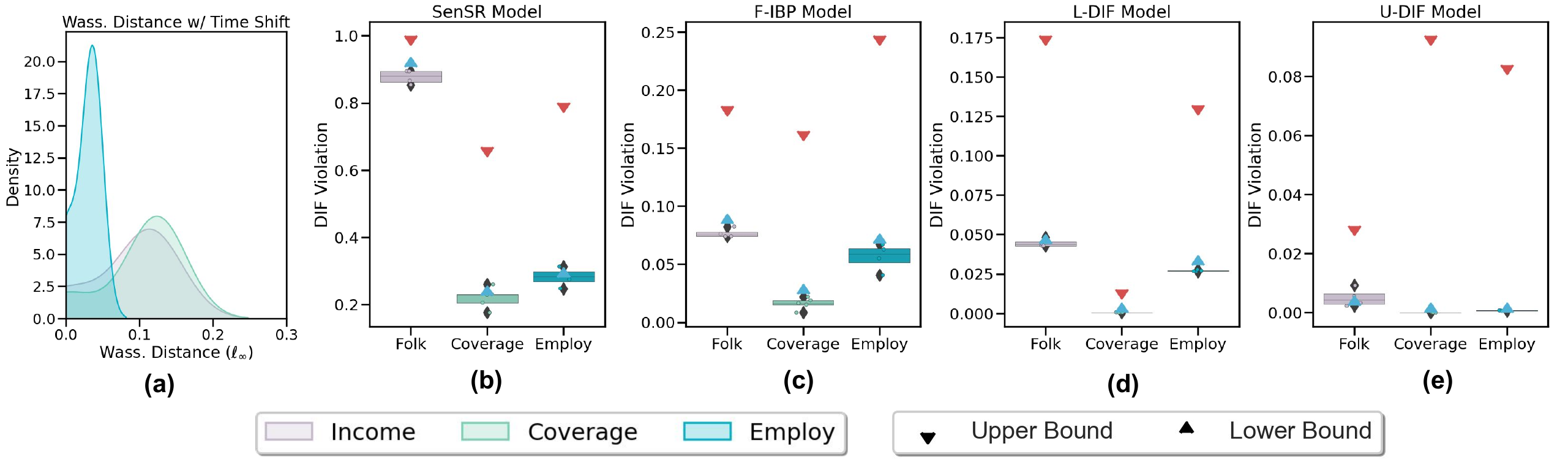}
    \vspace{-1em}
    \caption{Evaluating our bounds versus real-world distribution shifts. \textbf{Column (a):} An empirical distribution of Wasserstein distances between the distribution of individuals from different pairs of years, we certify w.r.t the upper quartile of these distributions. \textbf{Columns (b) - (e):} We plot the local fairncess certificates (LFC) for each of the shifted dataset using a boxplot. We then plot our lower bound on the worst-case DIF violation as a blue triangle and our upper bound on the worst-case DIF violation as a red triangle.}
    \label{fig:timeshift}
\end{figure}

\begin{figure}[h]
    \centering
\includegraphics[width=0.42\textwidth]{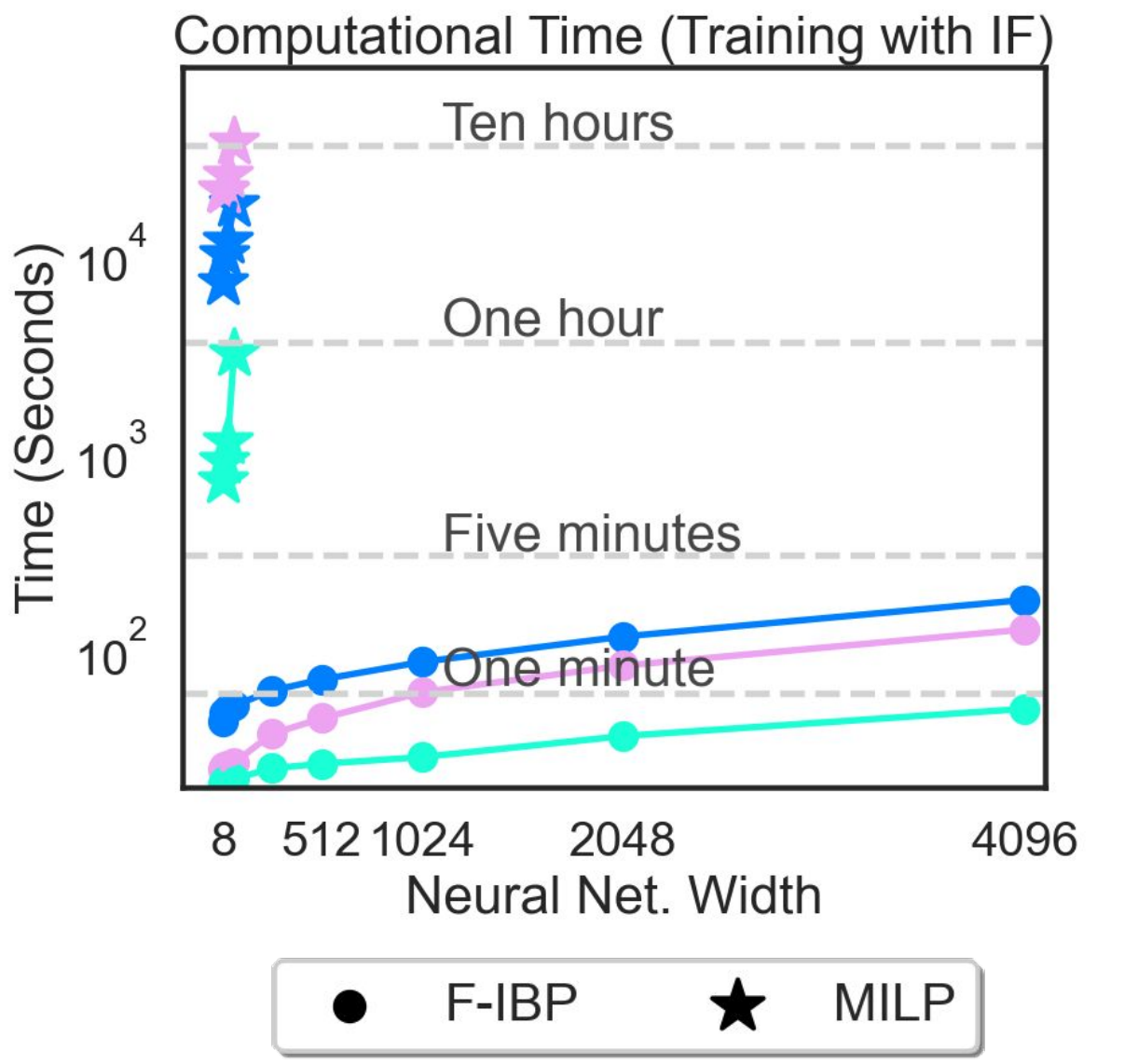}
    \caption{Computational time comparison between MILP training times reported in \citep{benussi2022individual} (plotted as stars) and F-IBP (plotted as circles) demonstrates our methods considerable advantage in scalability.}
    \label{fig:fibptimings}
\end{figure}

\section{Detailed Computations}\label{appendix:computations}

In this section we first describe how fair distance metrics are computed. We then provide the details for the exact computations of IBP and finally, an analysis of the bound presented in Theorem~\ref{lem:intervalboundlemma}.

\subsection{Computing Fair Distance Metrics}

In this subsection we breifly describe methods to compute $d_{\text{fair}}$ metrics. Below, we describe two popular forms of $d_{\text{fair}}$ and how they are computed. In this work, we assume that the $d_{\text{fair}}$ is always given by the SenSR metric.

\paragraph{Weighted $\ell_p$ metric} %\VP{MW: Make it crisper. Say exactly what we do (instead of what 'one' does), i.e. correlation coef etc.} 
A weighted $\ell_p$ metric for a vector $x$ in $\mathbb{R}^n$ takes the form: $\big(\sum_{i=0}^{n-1} \phi_i x^p \big)^{1/p}$ where $\phi$ is the weight vector. In \cite{john2020verifying} the $\phi_i$ is set to 0 for sensitive features and 1 for non-sensitive features. In this work, we attempt to capture the intra-correlation between sensitive and non-sensitive features by setting the weight of non-sensitive features $i$ to be $\phi_i = 1/|\rho_{i,j}|$ where $\rho_{i,j}$ is the Pearson correlation coefficient between the feature $i$ and the sensitive feature $j$.  %For IF, \st{one} takes $\phi_i = 0$ for sensitive features so individuals differing in these features are considered identical under the norm. To account for \st{intra}-feature correlation \st{one} often selects \st{$\phi_i \ll 1$} for all dimensions $i$ that have high correlation with sensitive attributes, which can be calculated on the training set.

\paragraph{SenSR metric}
%\VP{MW: Please see my annotations. Also tease on how exactly we learn the classifier}
While weighted $\ell_p$ metrics are intuitive they may be seen as overly simple. In \cite{yurochkin2019training}, the authors propose the SenSR metric (abbreviated as SR). Assuming only a single sensitive attribute, the SR metric is a Mahalanobis metric computed by first learning a logistic regression model to predict the sensitive attribute ($x_{\text{sens}}$) from the $n-1$ non-sensitive attributes ($x_{\text{nonsens}, i} \forall i \in [n-1]$), e.g., $\hat{x}_{\text{sens}} = exp(a^{T}_j x_{\text{nonsens}} + b_j) / \sum_{j=0}^{K}exp(a^{T}_j x_{\text{nonsens}} + b_j) $. Taking each vector $a_j$ to be the column of a matrix $A$, we have that the \textit{sensitive subspace} matrix $S$ is given by $S = I - P_{ran(A)}$ where $P_{ran(A)}$ is the orthogonal projector of the span of $A$. We then take $S$ to be the matrix for a fair Mahalanobis metric $d_{\text{fair}} = d_{S}(x, y) = \sqrt{(x-y)^\top S^{-1}(x-y)}$.

\subsection{IBP Computations}

Given a NN $f^{\theta}$ and an interval $[x^{L}, x^{U}]$ computed according to Theorem~\ref{lem:intervalboundlemma}, interval bound propagation proceeds by computing an interval over outputs $[y^L, y^U]$ such that $\forall x \in [x^{L}, x^{U}], \ y^{L} \leq f^{\theta}(x) \leq y^U$. Ultimately, computing the largest difference inside of $[y^{L}, y^U]$ will allow us to certify local IF. The output bounds, $[y^{L}, y^U]$, can be computed by performing a forward pass through the neural network with interval bound propagation (IBP). We adopt the notation for IBP proposed in \cite{gowal2018effectiveness} where $z^{k, L}$ and $z^{k, U}$ are the lower and upper bound on the inputs to the $k^{th}$ layer of the neural we can propagate this interval from layer $k$ to $k+1$ as follows:
\begin{align}
    z_\mu^{k} &= (z^{k, L} + z^{k, U})/2, \quad  
    z_r^{k} = (z^{k, U} - z^{k, L})/2 \\
    \zeta_\mu^{k} &= W^{k+1}z^{k}_\mu + b^{k+1},  \quad
    \zeta_r^{k} = |W^{k+1}|z^{k}_r \\
    z^{k+1, L}& = \  \sigma(\zeta_\mu^{k} - \zeta_r^{k}), \quad \ 
    z^{k+1, U} = \sigma(\zeta_\mu^{k} + \zeta_r^{k}) 
\end{align}
By passing the input interval $[x^L, x^U]$ through the above equations, as $z^{0,L}$ and $z^{0,U}$ respectively, we arrive at sound upper and lower bounds of the logits of the network. For the final activation, in our case the softmax, we can compute the lower and upper bounds on the softmax output for class $i$ with:
\begin{align}\label{eq:logitbounds}
\sigma^{K,L}_i &=
   \dfrac{e^{z^{K,L}_i}}{e^{z^{K,L}_i} + \sum_{j \neq i}e^{z^{K,U}_j}},\\
   \sigma^{K,U}_i &= \dfrac{e^{z^{K,U}_i}}{e^{z^{K,U}_i} + \sum_{j \neq i}e^{z^{K,L}_j}}
\end{align}
Finally, we have that for a $c$-class classification network that  $\max_{i \in [c]} \big(\sigma^{K,U}_i - \sigma^{K,L}_i\big) \leq \epsilon$ implies that local individual fairness is satisfied and constitutes a local individual fairness certificate for the network $f^{\theta}$ at $x$. 

\subsection{IBP Analysis}\label{appendix:IBPanalysis}

Any non-exact approximation of S can grow exponentially loose as the number of positive eigenvalues increases.
For instance, even if we approximate S with a hyper-rectangle aligned with the major axes of S, as done in \citep{benussi2022individual, ruoss2020learning}, the ratio of volume of the $\delta$ Mahalanobis ball ($\{x \in \mathbb{R}^n\ |\ d_S(x) \leq \delta \}$) and volume of hyper-rectangle is given by $\frac{\pi^{d/2}}{\Gamma(d/2+1)}\prod_{j | \lambda_j>0}\lambda_j$ to $4^{d/2}\prod_{j | \lambda_j>0}\lambda_j$, where $\lambda_{j}$ is the $j^{th}$ eigenvalue of $S$, and $d$ is the number of positive eigenvalues of S. It is not hard to see that the ratio of volumes decrease super-exponentially. Our method, an approximation aligned with the cannonical axis shares this quality. 
%this for our orthotope approximation empirically through Monte-Carlo estimate of volume of regions. 
We confirm the super-exponential decrease in volume ratio with Monte-Carlo estimation in Figure~\ref{fig:mh_tightness}.  Since this ratio decreases super-exponentially, the constraint imposed by an approximating orthotope could be too stringent when $d$ is large. % In our experiments, although the number of input dimensions is large, the sensitive features and therefore the numerical rank is only less than five (4, 2, 2 for German, Adult and Credit datasets respectively).

\begin{figure}
    \centering
    \includegraphics[width=0.45\textwidth]{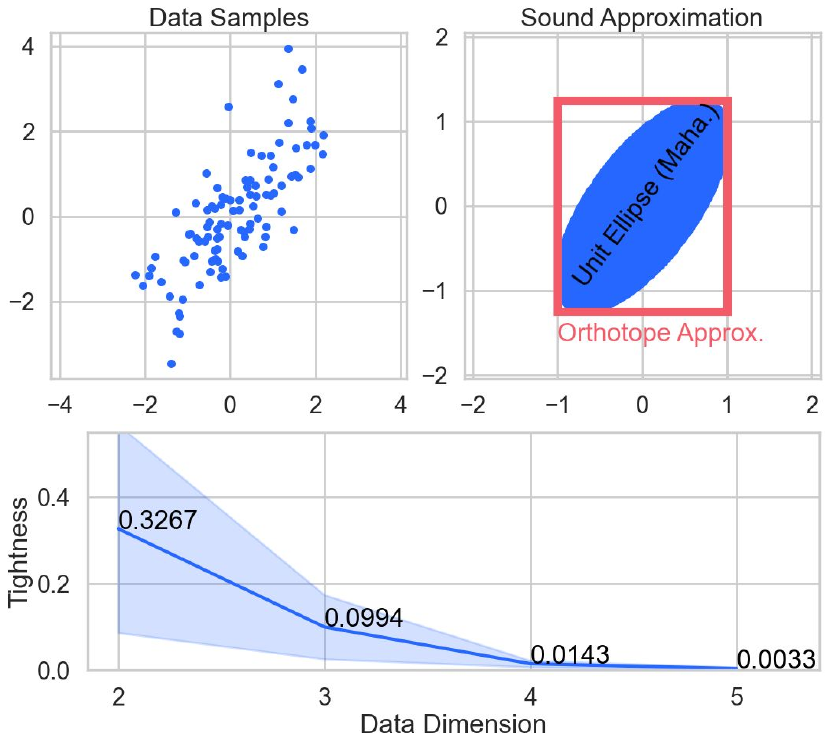}
    \caption{Visualization of our orthotope approximation. \textbf{Top Left:} Synthetic 2D data samples. \textbf{Top Right:} Unit ellipse with respect to the Mahalanobis distance computed on the data (blue) and our over-approximation of this ellipse (red). \textbf{Bottom:} The tightness (ratio of unit ellipse volume and approximation volume) computed for 20 random matrices with different dimensions. The super-exponential decrease of this approximation is discussed in Section~\ref{appendix:IBPanalysis}.
    }
    \label{fig:mh_tightness}
\end{figure}

\section{Proofs}\label{appendix:proofs}

In this section, we first prove Theorem~\ref{lem:intervalboundlemma}. We then show how Equation (1) and Equation~\eqref{eq:optimproblem} are equivalent given that the Wasserstein ball is over Dirac measures. Next, we provide a straight-forward proof that a global optima for Equation~\eqref{eq:upperbound} is a sound upper bound on Equation~\eqref{eq:optimproblem}. Lastly, we prove that the local IF function, $\mathcal{I}$ is H\"{o}lder continuous, therefore suitably chosen subgradient methods should converge to the global solution. 

\subsection{Proof of Theorem 5.1}
\begin{theorem}
Given a positive semi-definite matrix $S \in \mathbb{R}^{m \times m}$, a feature vector $x' \in \mathbb{R}^{m}$, and a similarity theshold $\delta$, all vectors $x'' \in \mathbb{R}^n$ satisfying $d_S(x', x'') \leq \delta$ are contained within the axis-aligned orthope:
$$ \Big[ x' - \delta\sqrt{d}, x' + \delta\sqrt{d} \Big] $$
where $d = diag(S)$, the vector containing the elements along the diagonal of $S$.
\end{theorem}
%\vspace{-1em}
\begin{proof}
Consider the case of $\delta = 1$. The width of the desired interval along the $i^{th}$ dimension can be obtained by solving $\max_{x''} e_i^T(x'' - x')\ s.t.\ d_{S}(x', x'') \leq 1$ where $e_i$ is the $i^{th}$ canonical basis vector. Let $S = R^T R$ where $R$ is the matrix square root of $S$, we can re-write the optimization in the rotated space by change of variables, $u:=R^{T}(x'' - x')$. W.r.t. $u$ the optimization becomes $\max_{u}R^T_i u \ \  s.t.\ ||u|| \leq 1$ where $R_i$ is the $i^{th}$ column of $R$. The solution of this optimization problem is then $R_i^T R_i / ||R_i|| = \sqrt{R_i^T R_i} = \sqrt{S}_{i,i}$. Generalizing to each dimension $i$ we have the bound $\sqrt{diag(S)}$, as desired. Because $S^{-1}$ is a linear transformation this bound remains sound when scaled by $\delta$ or translated to be centered at an arbitrary feature vector $x$. Proof of results related to Theorem~\ref{lem:intervalboundlemma} can be found in ~\cite{emrich2013optimal, MahaBound}.
\end{proof}

\begin{lemma}
Computing IBP w.r.t. the interval $ [ x' - \delta\sqrt{d}, x' + \delta\sqrt{d}] $ from Theorem~\ref{lem:intervalboundlemma} results in a sound over-approximation of the individual fairness violation as defined in Definition~\ref{def:IFdef}.
\end{lemma}
\begin{proof}
IBP produces output bounds $[y^{L}, y^{U}]$ such that $\forall x'' \in [x' - \delta\sqrt{d}, x' + \delta\sqrt{d}], f^{\theta}(x'') \in [y^{L}, y^{U}]$. By taking $\overline{\epsilon} = |y^{U} - y^{L}|$, we can certify that $\forall x'' \ s.t. \ d_{\text{fair}}(x', x'') \leq \delta \implies |f^{\theta}(x') - f^{\theta}(x'')| \leq \overline{\epsilon}$ which is exactly the condition needed to prove that local individual fairness holds according to Definition~\ref{def:IFdef}.
\end{proof}

\subsection{Equivalence between Equation (1) and Equation~\eqref{eq:optimproblem} for Dirac Measures}

Here we briefly describe the equivolence between Equation (1) and Equation~\eqref{eq:optimproblem}  when both distributions are Dirac measures, i.e., distributions given by a set of observed samples $\{x^{(i)}\}_{i=1}^{n}$. In this case, the reference distribution, $\hat{P}$, has support $\{\boldsymbol{\delta}_{x^{(i)}}\}_{i=1}^{n}$ where $\boldsymbol{\delta}_{x^{(i)}}$ is the Dirac function at location $x^{(i)}$, and we assume each observed sample has equivalent probability $1/n$. We then would like to measure the $p-$Wasserstein distance between $\hat{P^{0}}$ and another distribution. We make the assumption that the other distribution is also a Dirac measure. In this case, the $p-$Wasserstein distance between $\hat{P}$ and another Dirac measure $\hat{Q}$ with samples $\{z^{(i)}\}_{i=1}^{n}$ is known to be given by: 
$$ W_p (\hat{P}, \hat{Q}) := \inf_{\pi} \Big(\sum_{i=1}^{n} || x^{(i)} - z^{\pi(i)} ||^p \Big)^{1/p} $$
where $\pi$ is the set of all permutations of the set $[n]$. Without loss of generality, one can express $\hat{Q}$ as being given by a set of perturbation vectors $\phi^{(i)}$ such that the elements of $\hat{Q}$ are defined to be $\{ x^{(i)} + \phi^{(i)} \}_{i=1}^{n}$. In this case, it is clear that the $p-$Wasserstein distance is simply $\sum_{i=1}^{n}(||\phi^{(i)}||^p)^{1/p}$. Further, this makes it straight-forward to see that all distributions $\hat{Q}$ that are within a $p-$Wasserstein ball around $\hat{P}$ satisfy the constraint that $\sum_{i=1}^{n}(||\phi^{(i)}||^p)^{1/p} < \gamma$. Taking both sides to the $p^{th}$ power, we notice that this is precisely the constraint in Equation (1). Thus, the $\hat{Q}$ maximizing Equation~\eqref{eq:optimproblem} corresponds to the Dirac measure maximizing Equation (1), as desired.

\paragraph{Alternate and general proof}
\begin{theorem}
    Inputs from a distribution that is within $\gamma$ p-Wasserstein distance from the source empirical distribution $\hat{P}(\vx)=\sum_{i=1}^n\delta(\vx=\vx_i)$ must be contained in the region given by $\sum_{i=1}^n B(\vx_i, \delta_i)$ where $\delta_i>0, \sum_{i=1}^n\delta_i^p\leq \gamma^p$. In other words, the constraint in~\eqref{eq:optimproblem} over-constraints and therefore satisfies the constraints of (1). % ~\eqref{eqn:difconstraint1}.
\end{theorem}
\begin{proof}
Target distributions $Q$ that are $\gamma$ distance within the source distribution must satisfy the following constraint for all possible joint distributions: $\tau(\hat{P}, Q)$ over distributions $\hat{P}, Q$.
\begin{align*}
W_p(\hat{P}, Q) = \left(\inf_{\kappa\sim \tau(\hat{P}, Q)} \mathbb{E}_{(x, y)\sim \kappa}[d(x, y)^p]\right)^{1/p}
\end{align*}
Since $\hat{P}$ is an empirical distribution, minimum over joint distributions is obtained when an instance $y$ sampled from $Q$ is coupled with a closest point in $\hat{P}$ denoted $q(y, \hat{P})=\argmin_{i\in[1, n]} d(\vx_i, y)$.
We therefore have the following simplified expression for $W_p$.
\begin{align*}
W_p(\hat{P}, Q) = \mathbb{E}_{y\sim Q}[d(q(y, \hat{P}), y)^p]^{1/p}
\end{align*}
We denote by $\delta_i$ the maximum distance between $\vx_i$ and any point that is sampled from Q that is coupled with $\vx_i$, i.e.
\begin{align*}
    \delta_i \triangleq \max_{y\sim Q} d(q(y, \hat{P})=\vx_i, y)
\end{align*}
The definition of $\delta_i$ leads to an upper bound on the $W_p$ distance as follows. 
\begin{align*}
    W_p(\hat{P}, Q) &= \mathbb{E}_{y\sim Q}[d(q(y, \hat{P}), y)^p]^{1/p}\\
    &\leq \sum_i \left(\Pr(q(y, \hat{P})=\vx_i)\delta_i^p\right)^{1/p}\\
    &\leq \left((\sum_i \Pr(q(y, \hat{P})=\vx_i))(\sum_i \delta_i^p)\right)^{1/p}\\
    &=(\sum_i \delta_i^p)^{1/p}
\end{align*}
Following the definition of $\delta_i$, the support of Q (i.e. set of all points with non-zero probability) must be contained in $\cup_{i=1}^n B(\vx_i, \delta_i)$, which in itself is contained in $\sum_i B(\vx_i, \delta_i)$ with the additional constraint that $\sum_i \delta_i^p\leq \gamma^p$
\end{proof}

\subsection{Proof that Global Maximum of Equation~\eqref{eq:upperbound} is a Certificate}\label{appendix:upperboundproof}

Recall:
\paragraph{Theorem \ref{thm:upperboundthm}}
  Given an optimal assignment of $\{\varphi^{(i)} \}_{i=1}^n$ in Equation~\eqref{eq:upperbound}, the corresponding $\overline{\epsilon}$ is a sound upper-bound on the DIF violation of the model and therefore, is a certificate that no $\gamma$-Wasserstien distribution shift can cause the individual fairness of the model to exceed $\overline{\epsilon}$.

\begin{proof}
Let the set of vectors $\{\phi^{\star(i)}\}_{i=1}^{n}$ represent the global maximizing assignment of Equation~\eqref{eq:optimproblem}. Let $\{\varphi^{\star(i)}\}_{i=1}^{n}$ be the set of real values such that $\varphi^{\star(i)} = ||\phi^{\star(i)}||$.
From Theorem \ref{thm:upperboundthm} we have that for each $i$ that $\mathcal{I}(f^{\theta}, x^{(i)} + \phi^{\star(i)}, \delta) \leq \overline{\mathcal{I}(f^{\theta}, x^{(i)}, \delta + \varphi^{\star(i)})}$, 
thus we have that $\sum_{i=1}^{n} \mathcal{I}(f^{\theta}, x^{(i)} + \phi^{\star(i)}, \delta) \leq \sum_{i=1}^{n} \overline{\mathcal{I}(f^{\theta}, x, \delta + \varphi^{\star(i)} )}$. Given that $\{\varphi^{\star(i)}\}_{i=1}^{n}$ is a feasible assignment of the optimization problem and that it upper-bounds the maximum of Equation~\eqref{eq:optimproblem}, we have that the global maximum of Equation~\eqref{eq:upperbound}, must be an upper bound on Equation~\eqref{eq:optimproblem}. This is due to the fact that either  $\{\varphi^{\star(i)}\}_{i=1}^{n}$ is the maximizing assignment, or a the global maximum returns a value larger than that returned by $\{\varphi^{\star(i)}\}_{i=1}^{n}$ which would also be an upper bound to Equation~\eqref{eq:optimproblem}.
\end{proof}

\subsection{Proof $\mathcal{I}$ is H\"{o}lder Continuous (therefore globabally converges)}

Let $\vx_i$ denote a training instance in the original distribution, $\vx^{(1)}_i, \vx^{(2)}_i$ denote the points in the neighbourhood of $\vx_i$ such that 
\begin{align*}
\vx^{(1)}_i &= \vx_i + \varphi_i\\
\vx^{(2)}_i &= \argmax_{\vx\in B_\delta(\vx^{(1)}_i)} |f(\vx) - f(\vx^{(1)}_i)|.
\end{align*}
We have the following inequality for the absolute difference between function values at $\vx^{(1)}_i$ and $\vx^{(2)}_i$. 
\begin{align*}
    |f^\theta(\vx^{(1)}_i) - f^\theta(\vx^{(2)}_i)| &\leq     |f^\theta(\vx^{(1)}_i) - f^\theta(\vx_i) + f^\theta(\vx_i) - f^\theta(\vx^{(2)}_i)|\\
    \leq& |f^\theta(\vx^{(1)}_i) - f^\theta(\vx_i)| + |f^\theta(\vx_i) - f^\theta(\vx^{(2)}_i)|\\
    \leq & 2I[f^\theta; \vx_i, \varphi_i + \delta] \\
    \text{which follows because } & |\vx^{(2)}_i - \vx_i| \leq |\vx^{(1)}_i - \vx_i| +  |\vx^{(2)}_i - \vx_i|\leq \varphi_i + \delta\\
    \text{where I is as defined above }\\\text{i.e.}
    I[f^\theta; \vx, \delta] \triangleq & \max_{\hat{\vx}\in B_\delta(\vx)} |f^\theta(\hat{\vx}) - f^\theta(\vx)|\\ 
\end{align*}

Note that the inequality $|f^\theta(\vx^{(1)}_i) - f^\theta(\vx^{(2)}_i)|\leq 2I[f^\theta; \vx_i, \varphi_i + \delta]$ holds for any arbitrary value of $\vx_i$ and $\varphi_i$. 

Therefore eq3a becomes, 
\begin{align*}
    I(f^\theta; \vx^{(1)}_i, \delta)&\leq 2I[f^\theta; \vx_i, \varphi_i+\delta]\\
    \implies \frac{1}{n}\sum_{i=1}^n I(f^\theta; \vx^{(1)}_i, \delta) & \leq \frac{2}{n}\sum_{i=1}^n I(f^\theta; \vx_i, \varphi_i+\delta)
\end{align*}
The upper bound on the RHS can be obtained easily since $g_i(\varphi_i) = I(f^\theta; \vx_i, \delta+\varphi_i)$ is quasiconvex, which is easily seen by noting that any monotonic function is quasiconvex. 

Moreover, if f is C-lipschitz, then $\sum_i g_i(\varphi_i)$ satisfies Holder condition of order p for p=1. The proof is as follows. 

\begin{align*}
    |\sum_i g(\varphi_i^{(1)}) - \sum_i g(\varphi_i^{(2)})| &\leq \sum_i |I(f^\theta; \vx_i, \varphi_i^{(1)}) - I(f^\theta; \vx_i, \varphi_i^{(2)})|\\
    &= \sum_i |f^\theta(\vx_i^{(1)}) - f^\theta(\vx_i)| - I(f^\theta; \vx_i, \varphi_i^{(2)})\\
    \text{where WLOG, we assume }& \varphi_i^{(2)}\leq \varphi_i^{(1)}\\
    & \leq \sum_i |f^\theta(\widehat{\vx_i^{(1)}})\pm C(\varphi_i^{(1)}-\varphi_i^{(2)}) - f^\theta(\vx_i)| - I(f^\theta; \vx_i, \varphi_i^{(2)})\\
    &\text{where }\widehat{\vx_i^{(1)}}=\mathbb{P}_{\varphi_i^{(2)}}[\varphi_i^{(1)}]\\
    &\leq \sum_i|f^\theta(\widehat{\vx_i^{(1)}}) - f^\theta(\vx_i)| - I(f^\theta; \vx_i, \varphi_i^{(2)}) + C(\varphi_i^{(1)}-\varphi_i^{(2)})\\
    &\leq \sum_i I(f^\theta; \vx_i, \varphi_i^{(2)}) - I(f^\theta; \vx_i, \varphi_i^{(2)}) + C(\varphi_i^{(1)}-\varphi_i^{(2)})\\
    & = \sum_i C(\varphi_i^{(1)}-\varphi_i^{(2)})\\
    & = C|\varphi^{(1)} - \varphi^{(2)}|\\
    \text{where } \vx_i^{(1)} &= \argmax_{\vx\in B(\vx_i, \varphi_i^{(1)})} |f^\theta(\vx) - f^\theta(\vx_i)|\\
    \text{and } \mathbb{P}_{\varphi}(\vx) & \text{ is an operator to project a vector $\vx$ into $B(\vx, \varphi)$}
\end{align*}

Bounded optimization of quasiconvex functions that satisfy Holder condition of order p is shown to converge to global optimum when using exact or in-exact subgradient optimization methods in \citet{hu2020convergence}.

%\section{NeurIPS 2023 Checklist}

\end{document}